\documentclass{article}
\usepackage{amsthm}
\usepackage[utf8]{inputenc} 
\usepackage{geometry}
\usepackage{graphicx} 
\usepackage{booktabs} 
\usepackage{array} 
\usepackage{paralist} 
\usepackage{verbatim} 
\usepackage{subfig}
\usepackage{amsmath}
\usepackage{amsfonts}
\usepackage{amssymb}
\usepackage{amsthm}
\usepackage{authblk}
\usepackage{appendix}

\usepackage{fancyhdr}
\usepackage[numbers,sort&compress]{natbib}

\usepackage{sectsty}
\allsectionsfont{\sffamily\mdseries\upshape} 
\usepackage{url}
\usepackage[colorlinks=true, linkcolor=blue, citecolor=blue, urlcolor=blue]{hyperref}
\usepackage{thmtools}
\declaretheoremstyle[headfont=\bfseries, bodyfont=\normalfont]{proposition}

\usepackage[nottoc,notlof,notlot]{tocbibind} 
\usepackage[titles,subfigure]{tocloft}

\theoremstyle{remark}
\newtheorem{remark}{Remark}
\theoremstyle{plain}
\newtheorem{theorem}{Theorem}
\theoremstyle{proposition}

\theoremstyle{plain}
\newtheorem{lemma}{Lemma}
\theoremstyle{plain}

\theoremstyle{definition}
\newtheorem{defin}{Definition}
\theoremstyle{remark}
\newtheorem{example}{Example}
\theoremstyle{plain}
\newtheorem{ass}{Assumption}

\numberwithin{equation}{section}

\title{Optimal Non-Asymptotic Edgeworth Expansions for Multivariate Neural Network Outputs}
\author{Lucia Celli}
\date{} 
\affil[]{Department of Mathematics, University of Luxembourg}

\begin{document}

\maketitle
\begin{abstract}
    
Finite-width fully connected neural networks with Gaussian-initialized weights deviate from their infinite-width Gaussian limit, exhibiting non-vanishing higher-order cumulants. We approximate these deviations, for a neural network evaluated in a finite number of inputs, using multidimensional Edgeworth expansions of arbitrary order $4m-1$, with $m\in\mathbb{N}$. Assuming that the corresponding Gaussian limit has an invertible covariance matrix and that the activation function is polynomially bounded, we establish a bound of order $n^{-m}$ on the total variation distance between the law of the true network output and its Edgeworth approximation, with matching lower bounds. As an application, we quantify the error in Bayesian posterior distributions when the prior is replaced by its Edgeworth expansion. Our results are more general and also apply to sequences of conditionally Gaussian vectors converging to a Gaussian vector with invertible covariance.

\noindent{\bf Keywords:} {Edgeworth expansion; Neural networks; Limit theorems; Conditionally Gaussian Random variables; Gaussian initialization; Total variation distance; Bayesian supervised learning.}

\noindent{\bf AMS classification:} 60E10, 60F05, 60G15, 60G60, 68T07, 62C10
\end{abstract}

\section{Introduction}

It is well established that the output of a fully connected neural network with 
appropriately randomized initialization converges in distribution, as the layer 
width $n \to \infty$, to a Gaussian process 
(see, e.g., \cite{BR25,BT24,CP25,FHMNP,Trev,Torr23,Neal96,Klu22,BFF24,Han_Gas,Han22}). 
At finite width, however, networks deviate from this limit in systematic ways: their outputs exhibit non-Gaussian corrections whose structure and scaling can be explicitly quantified \cite{Han_Gas}.

{
A classical tool for approximating a random object that satisfies a Central Limit 
Theorem while accounting for its non-Gaussian fluctuations is the 
{Edgeworth expansion} (see e.g.\ \cite{H92,Mc87,BR86,Ko06,We13}). 
Its construction relies on the notion of {cumulants} (see e.g.\ \cite{NP12,PT11}), 
quantities given by the coefficients of the Taylor's expansion of the log-characteristic function {of a given random variable}. {These quantities provide a description of the shape of a distribution that is substantially finer than the sole use of mean and variance: for instance,} the third cumulant measures skewness, the fourth measures excess kurtosis, and higher-order cumulants capture increasingly subtle departures from Gaussianity.

In the CLT setting, under suitable moment conditions, {the relevant cumulants} 
decay in a controlled manner as the sample size grows, and the Edgeworth expansion 
organizes the correction to the Gaussian approximation as a power series in 
$n^{-1/2}$, where $n$ denotes the {sample size}.

Concretely, let $\mathbb{P}_n$ denote the law of the standardized sum $\sum_{i=1}^n\frac{X_n}{\sqrt{n}}$ for $\{X_i\}_i$ i.i.d. and let $\phi_1$ 
denote the standard Gaussian density on $\mathbb{R}$. Referring to \cite{MPS25,BR86}, the Edgeworth expansion takes 
the form
\begin{equation*}\label{eq:edgeworth}
    \int h(x)d\mathbb{P}_n(x) \;\approx\;\int h(x) \phi_1(x)\left(1 + \sum_{k=1}^{r} n^{-k/2}\, P_k(x)\right)dx,
\end{equation*}
for $h\in\mathcal{H}$, a class of test functions, and denoting, for each $k \geq 1$, $P_k$ is an explicit polynomial expressed in terms of 
Hermite polynomials (defined in \eqref{herm_def}), with coefficients determined 
by the cumulants of the summands. Truncating the series at order $r$ yields an 
approximation whose accuracy improves with both $r$ and $n$, and whose error can 
be rigorously controlled via suitable probability metrics 
(see e.g.\ \cite{MPS25,BR86}).
}

In recent years, several works have applied this Edgeworth-type of approximation to Bayesian neural networks. For instance, \cite{A19,NO24} introduce finite-width corrections of order $1/n$ to the Gaussian approximation for, respectively, single-hidden-layer and deep randomly initialized networks, explicitly exhibiting fourth-order Hermite terms. In particular, \cite{NO24} extend these results to the case of two inputs, introducing multidimensional Hermite polynomials. Building on these developments, \cite{Lu23} and \cite{NDSR20} construct a multivariate Edgeworth expansion (see, e.g., \cite{Mc87,H92}) up to the fourth cumulant for the joint distribution of finitely wide Bayesian neural-network outputs. Using this perturbed neural network as an approximation for the prior, the authors derive the corresponding non-Gaussian posterior distribution and apply it to Bayesian regression. These works thus propose a valid alternative to the Student's $t$ prior suggested, for example, by \cite{SWG14,PAPGGR23}, and to the infinite-width Gaussian prior, which has been shown to yield inferior performance (see, e.g., \cite{Ai20,PC21}).

These analyses strongly suggest that multidimensional Edgeworth expansions provide an accurate and tractable framework for capturing finite-width effects.
Nevertheless, what remains largely open is the derivation of non-asymptotic bounds (e.g. for the Total Variation distance between finite dimensional marginals) between the true law of a finite-width, multidimensional network output and its truncated Edgeworth expansion, including optimal rates and ideally matching lower bounds. In particular, while the works above show the plausibility of a multivariate Edgeworth scheme for fully connected neural networks, they do not deliver general, high-order error bounds of order $n^{-m}$ (for arbitrary $m\in\mathbb{N}$) together with matching lower bounds.

The aim of this paper is to fill this gap.  As an application, we use these results to quantify, in distributional metrics, the error incurred when approximating the posterior distribution of a neural network by replacing the prior with its Edgeworth expansion.

{
The code used to generate the figures in this paper is publicly available at~\cite{repo_edg_sim}.
}

\subsection{Structure of the paper}
In Section \ref{sec:tech_set} we introduce the formal framework of the paper, defining fully connected neural networks and conditionally Gaussian vectors. We also specify the Edgeworth expansion considered throughout and the total variation distance used to quantify approximation errors.

Section \ref{sec:our_contr} presents the main results, establishing upper and lower bounds on the total variation distance between the law of the neural network (respectively, the conditionally Gaussian vector) and its Edgeworth expansion in Theorem \ref{edg_nn} (respectively, Theorem \ref{edg_TV_generic}).

In Section \ref{sec:bayesian} we apply these results to a Bayesian supervised learning problem. In particular, we approximate the prior (given by the law of the neural network at initialization) by its Edgeworth expansion and estimate the resulting error, again in total variation distance, in the computation of the posterior distribution.

Appendix \ref{proof_th_edg_TV_gen} contains the proof of Theorem \ref{edg_TV_generic}, while Appendix \ref{proof_th_edg_TV} is devoted to the proof of Theorem \ref{edg_nn}. Additional technical lemmas are collected in the Appendix \ref{add_lemma}. Finally, in Appendix \ref{app:edg_exp} we briefly review the Edgeworth expansion following \cite{Mc87} and justify the choice of \eqref{edg_measure} as the appropriate Edgeworth expansion for conditionally Gaussian laws.

\section{Technical setup}\label{sec:tech_set}
\subsection{Notations}
For any random vector $X$ with values in $\mathbb{R}^d$ we write $X\sim\mathcal{N}_d(\mu,\Sigma)$ to denote that $X$ has a Gaussian distribution with expectation $\mu$ and with covariance matrix (resp. variance if $d=1$) $\Sigma$. When $\Sigma$ is definite positive (resp. different from zero if $d=1$) and $\mu=0$, we denote the density of $X$ as 
\[
x\in \mathbb{R}^d\to\phi_\Sigma(x).
\]

For any matrix $M\in\mathbb{R}^{d\times d}$ we denote by $\|M\|_{op}$ its operator norm and by $\|M\|_{HS}$ its Hilbert-Schmidt norm. For $n\in\mathbb{N}$ we write $M^{\otimes n}$ to denote the matrix in $\mathbb{R}^{(d\times d)^n}$ with components
\[
(M^{\otimes n})_{(i_1,\dots,i_{n}),(j_1,\dots,j_{n})}=M_{i_1,j_1}\dots M_{i_n,j_n}
\]
for $i_1,\dots,i_n,j_1,\dots,j_n=1,\dots,d$. We also denote $M^{\oplus n}$ the matrix in $\mathbb{R}^{nd\times nd}$, which is a diagonal block matrix with $n$ blocks that are equal to the matrix $M$.

\subsection{Fully connected neural networks and conditionally Gaussian matrices}
Let $\mathcal{X}:=\{x^{(1)},\dots,x^{(d)}\}\subseteq \mathbb{R}^{n_0}\setminus \{0\}$ be a collection of distinct inputs, and define
\begin{multline}\label{vec_z}
z^{(L+1)}(\mathcal{X})
:=\bigl(z_1^{(L+1)}(\mathcal{X}),\dots,z_{n_{L+1}}^{(L+1)}(\mathcal{X})\bigr)\\
:=\left(z_1^{(L+1)}(x^{(1)}),\dots,z_{1}^{(L+1)}(x^{(d)}),\dots, z_{n_{L+1}}^{(L+1)}(x^{(1)}),\dots,z_{n_{L+1}}^{(L+1)}(x^{(d)})\right)
\in \mathbb{R}^{n_{L+1}d}
\end{multline}
as the output vector of a fully connected neural network evaluated on $\mathcal{X}$.
For each $x\in\mathbb{R}^{n_0}$, the network is defined recursively (as in \cite{Han_Gas}) by
\begin{equation}\label{def_NN}
\begin{cases}
z_i^{(\ell)}(x)
= b_i^{(\ell)} + \displaystyle\sum_{j=1}^{n_{\ell-1}}
\sqrt{\frac{C_W}{n_{\ell-1}}}\, W_{i,j}^{(\ell)}\,
\sigma\!\left(z_j^{(\ell-1)}(x)\right),
& \ell=2,\dots,L+1,\\[1.2em]
z_i^{(1)}(x)
= b_i^{(1)} + \displaystyle\sum_{j=1}^{n_0}
\sqrt{\frac{C_W}{n_0}}\, W_{i,j}^{(1)}\, x_j,
& \ell=1,
\end{cases}
\end{equation}
for $i=1,\dots,n_\ell$. Here $\{b_i^{(\ell)}\}_{i,\ell}$ denote the biases, $\{W_{i,j}^{(\ell)}\}_{i,j,\ell}$ the weights, $\sigma:\mathbb{R}\to\mathbb{R}$ the activation function, $L$ the depth of the network, and $n_1,\dots,n_L$ its hidden-layer widths.

Following \cite{Han_Gas,CP25,FHMNP}, we impose the following assumptions.

\begin{ass}\label{bW}
At initialization, the biases $\{b_i^{(\ell)}\}_{i,\ell}$ are i.i.d.\ random variables with
\[
b_i^{(\ell)} \sim \mathcal{N}(0,C_b),
\]
for some constant $C_b\ge 0$, for all $\ell=1,\dots,L+1$ and $i=1,\dots,n_\ell$.
Moreover, the weights $\{W_{i,j}^{(\ell)}\}_{i,j,\ell}$ are i.i.d.\ random variables with
\[
W_{i,j}^{(\ell)} \sim \mathcal{N}(0,1),
\]
for all $\ell=1,\dots,L+1$, $i=1,\dots,n_\ell$, and $j=1,\dots,n_{\ell-1}$.
We further assume that the family of biases $\{b_i^{(\ell)}\}_{i,\ell}$ is independent of the family of weights $\{W_{i,j}^{(\ell)}\}_{i,j,\ell}$.
\end{ass}

\begin{ass}\label{hip_s}
    The activation function $\sigma$ is not constant and there exists an integer $r\ge 1$ such that $\sigma$ is either $r$ times continuously differentiable or it is $r-1$ times continuously differentiable and the $(r-1)$-derivative is a piece-wise linear function with a finite number of points of discontinuity for its derivative.

    Moreover the $r$-derivative of $\sigma$ is polynomially bounded, i.e. there exists $k\ge 1$ s.t.
    \[
    \sup_{x\in\mathbb{R}}\left|(1+|x|)^{-k}\frac{d^r}{dx^r}\sigma(x)\right|<\infty.
    \]
\end{ass}
Under Assumption \ref{bW}, the following Lemma \ref{gaus_stut_NN} is an immediate consequence of standard arguments, as observed, for example, in \cite{Han_Gas,FHMNP,CP25,Trev,BT24}, and constitutes a key ingredient of the present work. Before stating Lemma \ref{gaus_stut_NN}, we recall the definition of conditionally Gaussian vectors.

\begin{defin}[Conditional Gaussianity]\label{cond_Ga_def}
Let $Z$ be a square-integrable random vector
defined in $(\Omega,\mathcal A,\mathbb P)$ with values in $\mathbb R^{p}$, for $p\in\mathbb{N}$,
and such that $\mathbb E[Z]=0$.
We say that $Z$ is \emph{conditionally Gaussian} with respect to a $\sigma$-field
$\mathcal F\subseteq\mathcal A$ if there exists a positive semi-definite and
$\mathcal F$-measurable matrix
\[
A \in \mathbb R^{p\times p},
\]
called the \emph{conditional covariance matrix}, such that, $\mathbb P$-a.s.,
\begin{equation}\label{cond_Ga_def_tensor}
\mathbb E\!\left[
e^{i \langle y, Z\rangle}
\middle| \mathcal F
\right]
=
e^{-\frac12\langle y,Ay\rangle},
\end{equation}
for every $y\in\mathbb R^{p}$.
\end{defin}

\begin{lemma}[Lemma 7.1 in \cite{Han_Gas}, Lemma 1 in \cite{CP25}]\label{gaus_stut_NN}
    For every $\ell=1,\dots,L+1$ define $\mathcal{F}_\ell$ as the $\sigma$-field generated by $\{b_i^{(k)},W_{i,j}^{(k)}\}^{k=1,\dots,\ell}_{i=1,\dots, n_{k},j=1,\dots,n_{k-1}}$. Then, for every $\ell=2,\dots,L+1$, the following properties hold for the neural network evaluated in a finite number of inputs $($defined in \eqref{vec_z}$)$:
    \begin{enumerate}
        \item [(i)] conditionally on $\mathcal{F}_{\ell-1}$, the random vectors 
        \[
        \left\{z_i^{(\ell)}(\mathcal{X}):=\left(z_i^{(\ell)}(x^{(1)}),\dots,z_i^{(\ell)}(x^{(d)})\right)\right\}_{i=1,\dots,n_{\ell}}
        \] 
        are stochastically independent;
        \item[(ii)] the random vector $z^{(\ell)}(\mathcal{X})$ defined in \eqref{vec_z} is a conditionally Gaussian vector $($as in Definition \ref{cond_Ga_def}$)$ with respect to the $\sigma$-field $\mathcal{F}_{\ell-1}$ and its conditional covariance matrix is given by $(A^{(\ell)})^{\oplus n_{\ell}}$, defined as a diagonal block matrix with $n_\ell$ blocks and with the same matrix $A^{(\ell)}$ on every block of the diagonal, where $A^{(\ell)}$ is a random matrix with values in $\mathbb{R}^{d\times d}$ and components defined as
        \begin{equation}\label{cov_z}
        A^{(\ell)}_{i,j}:=
C_b+\frac{C_W}{n_{\ell-1}}\sum_{k=1}^{n_{\ell-1}}\sigma(z_k^{(\ell-1)}(x^{(i)}))\sigma(z_k^{(\ell-1)}(x^{(j)})).
        \end{equation}
        \end{enumerate}
\end{lemma}
Thanks to Lemma \ref{gaus_stut_NN}, we can focus on conditionally Gaussian vectors and derive general results that can be applied to the neural network defined in \eqref{def_NN} under Assumptions \ref{bW} and \ref{hip_s}.

\subsection{Edgeworth expansion for a conditionally Gaussian vector}
Let $Z$ be a conditionally Gaussian vector (with respect to a $\sigma$-field $\mathcal{F}$) taking values in $\mathbb{R}^{n d}$, for $n,d\in\mathbb{N}$, and with conditional covariance matrix $A^{\oplus n}$ (see Definition \ref{cond_Ga_def}), where $A^{\oplus n}$ denotes the block-diagonal matrix with $n$ diagonal blocks, each equal to a matrix $A\in\mathbb{R}^{d\times d}$.

{ Assuming that
\begin{equation}\label{hyp_mom_A}
\mathbb{E}[\|A\|_{HS}^{2m-1}]<\infty,
\end{equation}
}
out goal is to approximate the law of $Z$ by perturbing an initial Gaussian distribution
\begin{equation}\label{gaus_limit}
 \mathcal{N}_{n d}\!\left(0, K^{\oplus n}\right),
\end{equation}
where $K^{\oplus n}$ denotes the block-diagonal matrix with $n$ identical blocks equal to $K\in\mathbb{R}^{d\times d}$, which is assumed to be invertible. {The approximation is obtained via the corresponding Edgeworth expansion, as developed in \cite{Mc87}, and subsequently specialized to the case of conditional Gaussian vectors, following the approach detailed in Appendix~\ref{app:edg_exp}.
}

We define $\gamma_{Z,G,m}$ as the signed measure on $\mathbb{R}^{n d}$ whose density is given by the multivariate Edgeworth expansion of order $(4m-1)$ associated with $Z$:
{
\begin{multline}\label{edg_measure}
x\in\mathbb{R}^{nd}\longmapsto 
\prod_{i=1}^{n}\phi_K(\tilde{x}_i)
+\sum_{k=1}^{2m-1}\frac{1}{k!\,2^k}
\sum_{J\in S_{nd}^{(2k)}}\sum_{\alpha\in \mathcal{A}_J}
\mathbb{E}\!\left[(Q^{\oplus n})_{\alpha_1,\alpha_2}\cdots
(Q^{\oplus n})_{\alpha_{2k-1},\alpha_{2k}}\right]\cdot\\
\cdot \prod_{i=0}^{n-1}
\left(
H_{j_{di+1}}\!\left(u^{(i)}_1\right)\cdots
H_{j_{di+d}}\!\left(u^{(i)}_d\right)
\right)_{|_{\,u^{(i)}=\sqrt{K}^{-1}\tilde{x}_{i+1}}}\,
\phi_K\!\left(\tilde{x}_{i+1}\right),
\end{multline}
}
where $\tilde{x}_i:=(x_{d(i-1)+1},\dots,x_{d(i-1)+d})\in\mathbb{R}^d$,
\[
Q:=\sqrt{K}^{-1}(A-K)\sqrt{K}^{-1},
\]
and $\phi_K$ denotes the density of the Gaussian law defined in \eqref{gaus_limit}. Moreover, as in \cite{CP25}, we have 
\begin{equation}\label{def_S_k}
S_{nd}^{(2k)}:=\Bigl\{J:=(j_1,\dots,j_{nd})\in(\mathbb{N}\cup\{0\})^{nd}:
j_1+\cdots+j_{nd}=2k\Bigr\},
\end{equation}
{
i.e. the set of all multi-indices $ J = (j_1,\dots,j_{nd}) $ whose entries are non-negative integers and such that the sum of all components is equal to $ 2k $,
}
and
\begin{equation}\label{def_A_j}
\mathcal{A}_J:=
\Bigl\{\alpha=(\alpha_1,\dots,\alpha_{2k})\in\{1,\dots,nd\}^{2k}:
\sum_{r=1}^{2k}\mathbf{1}_{\{\alpha_r=s\}}=j_s
\ \forall\, s=1,\dots,nd\Bigr\},
\end{equation}
for $J\in S_{nd}^{(k)}$,
{
i.e. the set of all sequences $\alpha = (\alpha_1,\dots,\alpha_{2k}) $ whose entries take values in $ {1,\dots,nd} $ and such that each value $s \in \{1,\dots,nd\} $ appears exactly $j_s$ times in the sequence.
}
Finally, for every $j\in\mathbb{N}\cup\{0\}$, $H_j$ denotes the Hermite polynomial of degree $j$ (see e.g. \cite{NP12}), defined via the derivatives of the standard Gaussian density $\phi_1$ by
\begin{equation}\label{herm_def}
H_j(x):=\frac{(-1)^j}{\phi_1(x)}\frac{\mathrm{d}^j}{\mathrm{d}x^j}\phi_1(x).
\end{equation}
{

\begin{remark}
    For $J=(j_1,\dots,j_{nd})\in S_{nd}^{(k)}$, it is easy to observe that 
    \[
    \mathcal{A}_J=\left\{\eta\bigl(\underbrace{1,\dots,1}_{j_1},\dots,\underbrace{nd,\dots,nd}_{j_{nd}}\bigr), \eta\in \Pi_{nd}\right\},
    \]
    using the same notation as in Example \ref{ex_edg_dens}. In other words, the set $\mathcal{A}_J$ is given by all the possible permutations of a fixed sequence in which each value $s \in \{1,\dots,nd\}$ appears exactly $j_s$ times.

\end{remark}

\begin{remark}
    The measure associated with the density $\gamma_{Z,G,m}$ defined in \eqref{edg_measure} is a finite measure since
    \begin{multline*}
    \int_{\mathbb{R}^{nd}}|\gamma_{Z,G,m}(x)|dx\le 1+\sum_{k=1}^{2m-1}\frac{1}{k!2^k}\sum_{J\in S_{nd}^{(2k)}}\sum_{\alpha\in \mathcal{A}_J}
\mathbb{E}\!\left[\|Q\|_{op}^k\right]\cdot\\
\cdot \prod_{i=0}^{n-1}\mathbb{E}\left[
\left|
H_{j_{di+1}}\!\left(N^{(i)}_1\right)\cdots
H_{j_{di+d}}\!\left(N^{(i)}_d\right)
\right|\right],
    \end{multline*}
    where $\{N^{(i)}\}_{i}$ are independent and identically distributed as $\mathcal{N}(0,I_d)$ and $|Q_{i,j}|\le \|Q\|_{op}$ for every $i,j\in\{1,\dots,d\}$. Hence, using H\"older's inequality and the fact, presented in Proposition 1.4.2 in \cite{NP12}, that
    \[
    \mathbb{E}\left[H_j(N)^2\right]=j!\quad \text{for every $j\in \mathbb{N}$ and with $N\sim\mathcal{N}(0,1)$}, 
    \]
    it results that
    \[
    \int_{\mathbb{R}^{nd}}|\gamma_{Z,G,m}(x)|dx\le 1+\sum_{k=1}^{2m-1}\frac{1}{k!2^k}\sum_{J\in S_{nd}^{(2k)}}\sum_{\alpha\in \mathcal{A}_J}
\mathbb{E}\!\left[\|Q\|_{op}^k\right] \prod_{i=0}^{n-1}(j_{di+1}!\dots j_{di+d}!)^{1/2}<\infty,
    \]
    since all the sums involved have a finite number of terms.
\end{remark}

\begin{example}\label{ex_edg_dens}
For $m=1$, we have that the density $\gamma_{Z,G,m}$ defined in \eqref{edg_measure} reads as
\begin{multline*}
 \gamma_{Z,G,m}(x)=  \prod_{i=1}^{n}\phi_K(\tilde{x}_i)
+\frac{1}{2}
\sum_{J\in S_{nd}^{(2)}}\sum_{\alpha\in \mathcal{A}_J}
\mathbb{E}\!\left[(Q^{\oplus n})_{\alpha_1,\alpha_2}\right]\cdot\\
\cdot \prod_{i=0}^{n-1}
\left(
H_{j_{di+1}}\!\left(u^{(i)}_1\right)\cdots
H_{j_{di+d}}\!\left(u^{(i)}_d\right)
\right)_{|_{\,u^{(i)}=\sqrt{K}^{-1}\tilde{x}_{i+1}}}\,
\phi_K\!\left(\tilde{x}_{i+1}\right)
\end{multline*}
Moreover, denoting with $\Pi_{nd}$ the sets of all the permutations of $nd$ elements and writing $\eta(a_1,\dots,a_{nd}):=(a_{\eta(1)},\dots a_{\eta(nd)})$ for every vector $(a_1,\dots,a_{nd})\in\mathbb{R}^{nd}$ and $\eta \in \Pi_{nd}$, we obtain that
\[
S_{nd}^{(2)}=\left\{\eta\left(2,0,\dots,0\right)\right\}_{\eta\in\Pi_{nd}}\cup\left\{\eta\left(1,1,0\dots,0\right)\right\}_{\eta\in\Pi_{nd}}.
\] 
For any $J\in S_{nd}^{(2)}$ of the type $\eta(2,0,\dots,0)$ for $\eta \in \Pi_{nd}$, one has
\[
\mathcal{A}_J=\{(\eta(1),\eta(1))\}
\]
and for any $\bar{J}\in S_{nd}^{(2)}$ of the type $\eta(1,1,0,\dots,0)$ for $\eta \in \Pi_{nd}$, we have
\[
\mathcal{A}_{\bar{J}}=\{(\eta(1),\eta(2))\}.
\]
Therefore, 
\begin{multline*}
 \gamma_{Z,G,m}(x)=   \prod_{i=1}^{n}\phi_K(\tilde{x}_i)
+\frac{1}{2}
\sum_{i=0}^{n-1}\sum_{r=1}^d\sum_{s=1, s\ne r}^d
\mathbb{E}\!\left[Q_{r,s}\right]
\left(
u^{(i)}_r u^{(i)}_s\right)
_{|_{\,u^{(i)}=\sqrt{K}^{-1}\tilde{x}_{i+1}}}\,
\phi_K\!\left(\tilde{x}_{i+1}\right)\\
+\frac{1}{2}\sum_{i=0}^{n-1}\sum_{r=1}^d
\mathbb{E}\!\left[Q_{r,r}\right]
\left((u^{(i)}_r)^2-1\right)
_{|_{\,u^{(i)}=\sqrt{K}^{-1}\tilde{x}_{i+1}}}\,
\phi_K\!\left(\tilde{x}_{i+1}\right),
\end{multline*}
since 
\[
H_1(x)=x\quad\text{and}\quad H_2(x)=x^2-1\quad\text{for all $x\in\mathbb{R}$}.
\]

\end{example}
\begin{example}\label{es_edg_11}
    In the simple case of $n=1$ and $d=1$, the Edgeworth expansion can be written explicitly as follows:
    \[
    x\in\mathbb{R}\to \gamma_{Z,G,m}(x)=\phi_K(x)\left(1+\sum_{k=1}^{2m-1}\frac{\mathbb{E}[Q^k]}{k!2^k}H_{2k}\left(\sqrt{K}^{-1}x\right)\right),
    \]
    since in this case $Q=Q_{1,1}\in\mathbb{R}$ and 
    \[
    S_1^{(2k)}=\{(2k)\}\quad\text{and}\quad\mathcal{A}_J=\{(1,\dots,1)\}\subseteq\mathbb{N}^{2k}\quad\text{for $J\in S_1^{(2k)}$}.
    \]
    \end{example}
    \begin{example}\label{es_NN_edg}
    Take $Z=z^{(2)}_1(1)$, according to the definition for a fully connected neural network \eqref{def_NN} with $L=1,d=1,n_0=1,$ $x=1$ and parameters $C_W=\sqrt{2},C_b=0$ and $\sigma(x)=ReLU(x)=x1_{\{x\ge 0\}}$. Then take $G\sim\mathcal{N}(0,1)$ (observe that the variance $K^{(2)}$ defined in the following Theorem \ref{fin_th_n} satisfies $K^{(2)}=1$ in this particular setting). Then, thanks to Example \ref{es_edg_11}, one can explicitly compute the following approximation to the distribution of the neural network defined above:
    \begin{itemize}
    \item \begin{equation}\label{gas_dens_es}
    y\in\mathbb{R}\mapsto \phi_1(y),
    \end{equation}
        \item \begin{equation}\label{para_gauss_1_es}
        y\in\mathbb{R}\mapsto \gamma_{Z,G,1}(y)=\phi_{1}(y)\left(1+\frac{5}{8n_1}H_4\left({y}\right)\right),
        \end{equation}
        \item \begin{equation}\label{para_gauss_2_es}
        y\in\mathbb{R}\mapsto \phi_{1}(y)\left(1+\frac{5}{8n_1}H_4\left({y}\right)+\frac{11}{12n_1^2}H_6\left({y}\right)\right),
        \end{equation}
        \item \begin{multline}\label{para_gauss_3_es}
            y\in\mathbb{R}\mapsto \gamma_{Z,G,2}(y)\\
            =\phi_{1}(y)\left(1+\frac{5}{8n_1}H_4\left({y}\right)+\frac{11}{12n_1^2}H_6\left({y}\right)+\left(\frac{1573}{192n_1^2}
+\frac{25(n_1-1)}{64n_1^2}\right)H_8\left({y}\right)\right),
\end{multline}
    \end{itemize}
    where $\phi_1$ is the density of the standard Gaussian measure. In Figure \ref{fig:compare_NN_para} one can see the plots of the functions defined above and compare them with an approximation of the neural network's density.
    \begin{figure}[htbp]
    \centering
    \includegraphics[width=0.7\textwidth]{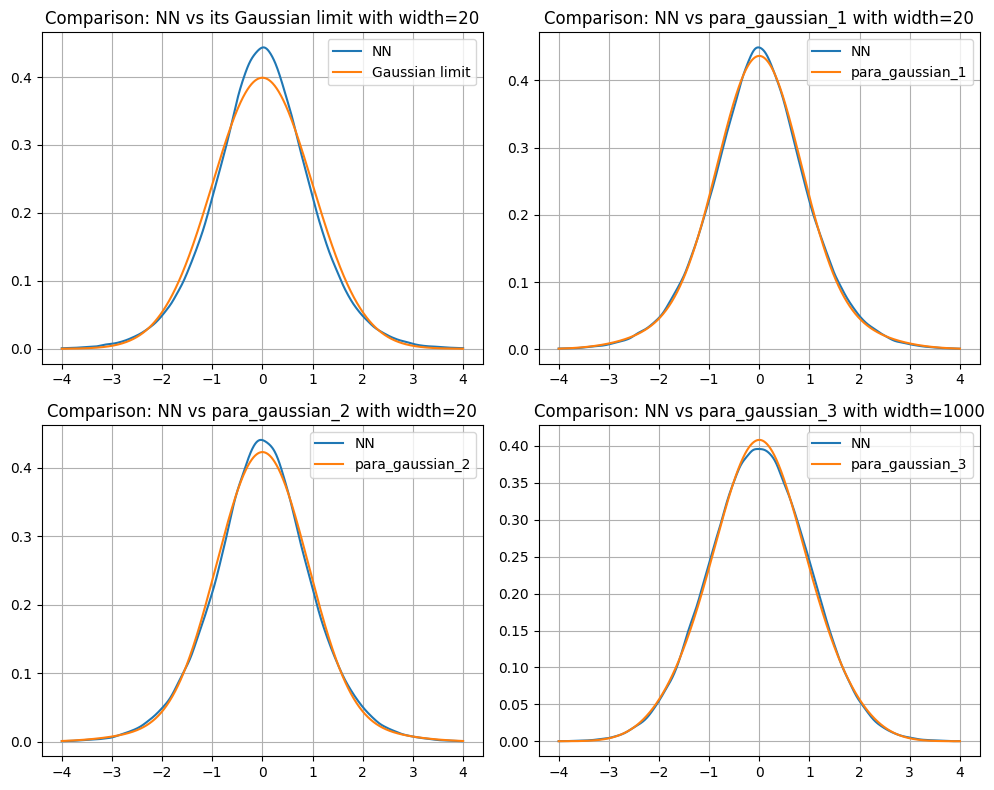}
    \caption{ Comparison between the estimated density of a neural network output (NN) as in Example \ref{es_NN_edg} and its approximations defined in \eqref{gas_dens_es}, \eqref{para_gauss_1_es}, \eqref{para_gauss_2_es},
 \eqref{para_gauss_3_es}. The top-left panel shows the comparison with the infinite-width Gaussian limit (approximated at width 20). The two panels on the right compare the NN with the Edgeworth expansion with $m=1$ and width equal to 20 (see \eqref{para_gauss_1_es}, panel on top), and with the Edgeworth expansion with $m=2$ and width equal to 1000 (see \eqref{para_gauss_3_es}, panel on the bottom). The panel on the left-bottom compare the distribution of the NN with the para-Gaussian defined in \eqref{para_gauss_2_es}, following the terminology of \cite{NO24}. The choice of increasing the width in the bottom-right panel is due to the fact that a lower value of the width would not compensate the high oscillations of the Hermite polynomial of order 8.}
    \label{fig:compare_NN_para}
\end{figure}
\end{example}

}

In the next section, we quantify the accuracy of the Edgeworth approximation of the neural network output by estimating the discrepancy between the law of the network and its Edgeworth expansion in total variation distance, defined below.

{
\begin{defin}[{Total variation distance}, see e.g. Appendix C in \cite{NP12}]
    Given two finite signed measures $\mu,\nu$  with values in $\mathbb{R}^{p}$, for $p\in\mathbb{N}$, the Total Variation distance between them is defined as
\begin{equation}\label{dTV_def}
d_{TV}(\mu,\nu):=\sup_{B\in\mathcal{B}(\mathbb{R}^{p})}\Big|\mu(B)-\nu(B)\Big|
=\frac{1}{2}\sup_{h\in \mathcal{M}_1}\Big|\int_{\mathbb{R}^p}h(x)d\mu(x)-\int_{\mathbb{R}^p}h(x)d\nu(x)\Big|,
\end{equation}
where $\mathcal{B}(\mathbb{R}^{p})$ is the Borel $\sigma$-field of $\mathbb{R}^{p}$ and {
\[
\mathcal{M}_1:=\{h:\mathbb{R}^{p}\to\mathbb{R}\quad\text{Borel measurable with}\quad \|h\|_{\infty}\le 1\}.
\]
}
\end{defin}
}
\section{Our contribution}\label{sec:our_contr}
Our first main result provides an upper bound on the Total Variation distance between the law of a conditionally Gaussian vector and its Edgeworth expansion. Additionally, we establish a lower bound in the case where the Edgeworth expansion is performed with respect to a specific reference Gaussian law. We prove the following Theorem in Section \ref{proof_th_edg_TV_gen}.

{
\begin{theorem}\label{edg_TV_generic}
    Let $Z$ be a centered conditionally Gaussian vector in $\mathbb{R}^{n d}$ with conditional covariance matrix $A^{\oplus n}$ and law $P_Z$, as in Definition \ref{cond_Ga_def}, and let $G$ be a centered Gaussian vector in $\mathbb{R}^{n d}$ with covariance matrix $K^{\oplus n}$ and $K\in\mathbb{R}^{d\times d}$ invertible. Assume that $\mathbb{E}[\|A\|_{op}^{4m}]<\infty$ and let $\gamma_{Z,G,m}$ be the signed measure defined in \eqref{edg_measure}. Then the Total variation distance (see Definition \ref{dTV_def}) between the law of the conditionally Gaussian vector $Z$ and its Edgeworth approximation $\gamma_{Z,G,m}$  is upper bounded as follows:
    \begin{equation}\label{bound_dTV_gen_no_exp}
    d_{TV}(P_Z,\gamma_{Z,G,m})\le  C\mathbb{E}\left[\|A-K\|_{\text{HS}}^{4m}\right]^{\frac{1}{2}},
\end{equation}
where 
$C>0$ is an explicit constant that depends only on the dimensions $d$ and $n$, on the order of approximation $m$ and on the minimum eigenvalue of $K$ $($see the upper bound \eqref{expl_bound_edg_TV_gen}$)$.

If $\mathbb{E}[A]=K$, then, for every index $j\in\{1,\dots,d\}$, we also have that 
\begin{equation}\label{lower_bound_egd_gen}
d_{TV}(P_Z,\gamma_{Z,G,m})\ge  \frac{\mathbb{E}\left[\left(A_{j,j}-K_{j,j}\right)^{2m}\right]e^{-\frac{K_{j,j}}{2}}}{2^{2m}(2m)!}-\frac{\mathbb{E}\left[\left|A_{j,j}-K_{j,j}\right|^{2m+1}\right]}{2^{2m+1}(2m)!}.
\end{equation}
\end{theorem}
}

The second main result is a direct application of Theorem \ref{edg_TV_generic} to the case when $Z$ is a fully connected neural network as in \eqref{def_NN} under Assumptions \ref{bW} and \ref{hip_s}. In fact, it has been proved by \cite{Han_Gas} (see also \cite{Neal96,MHRTG,LBNSPS}) that the neural network converges in law to a Gaussian process as the inner widths grow to infinite and \cite{Trev,BT24,CP25,FHMNP,Torr23,BFF24} provided quantitative results for this Central Limit Theorem. In particular, we focus on the following result by \cite{CP25}.
\begin{theorem}(Theorem 5, Remark 8 in \cite{CP25})\label{fin_th_n}
    Let Assumptions \ref{bW} and \ref{hip_s} hold. Fix a set of distinct inputs $\mathcal{X}:=\{x^{(1)},\dots,x^{(d)}\}\subseteq \mathbb{R}^{n_0}\setminus \{0\}$.

    Define the matrix $K^{(L+1)}\in\mathbb{R}^{d\times d}$ in a recursive way: for every $i,j\in\{1,\dots,d\}$
    \begin{equation}\label{lim_cov}
    K_{i,j}^{(\ell)}:=\begin{cases}
        C_b+C_W\mathbb{E}\Big[\sigma(G_1^{(\ell-1)}(x^{(i)}))\sigma(G_1^{(\ell-1)}(x^{(j)}))\Big] &\text{if $\ell\ge 2$}\\
        \\
        C_b+\frac{C_W}{n_0}\sum_{k=1}^{n_0}x^{(i)}_{k}x^{(j)}_{k}&\text{if $\ell= 1$},
    \end{cases}
    \end{equation}
    where $(G_i^{(\ell)})_{i=1,\dots,n_\ell}$ are independent for every $\ell=1,\dots, L+1$ and
    \[
\Big(G_1^{(\ell-1)}(x^{(i)}), G_1^{(\ell-1)}(x^{(j)})\Big)\sim\mathcal{N}_2\Bigg(0,
\begin{pmatrix}
K_{i,i}^{(\ell-1)} & K_{i,j}^{(\ell-1)}\\
K_{i,j}^{(\ell-1)} & K_{j,j}^{(\ell-1)}
\end{pmatrix}
\Bigg).
\]
Assume that the matrices defined in \eqref{lim_cov}, $\{K^{(\ell)}\}_{\ell=1,\dots,L+1}$, are invertible on $\mathcal{X}$. Then, if there exists $n\in\mathbb{N}$ such that 
{
\begin{equation}\label{bounds_n}
cn\le n_1, \dots n_{L}\le Cn
\end{equation}}
for some $c,C>0$ constants, and recalling the definition in \eqref{dTV_def}, one has that
\begin{equation}\label{e:tvmain}
d_{TV}\Big(z^{(L+1)}(\mathcal{X}),G^{(L+1)}(\mathcal{X})\Big)\le \frac{D}{n}  
\end{equation}
where $D$ is a positive constant that does not depend on $n, n_1,\dots, n_{L}$ and we denote $z^{(L+1)}(\mathcal{X})$ as in \eqref{vec_z} for every $i=1,\dots,n_{L+1}$ 

and similarly we denote
\[
G^{(L+1)}(\mathcal{X}):=\left(G_1^{(L+1)}(x^{(1)}),\dots,G_1^{(L+1)}(x^{(d)}),\dots,G_{n_{L+1}}^{(L+1)}(x^{(1)}),\dots,G_{n_{L+1}}^{(L+1)}(x^{(d)})\right).
\]
\end{theorem}
\begin{remark}
{
    As noted in Remark 10 of \cite{CP25}, the assumption that the matrices $\{K^{(\ell)}\}_{\ell}$ are invertible is not restrictive. Indeed, by Theorems 6 and 7 in \cite{CC24}, this is equivalent to imposing certain general conditions on the inputs $\mathcal{X}$, assuming that the activation function is continuous and non-polynomial. More in details, it is enough to assume that the inputs are all distinct if the variance of the bias, $C_b$, is different from zero, and, if $C_b\ne 0$, then it is sufficient to impose that the inputs are pairwise non proportional.
    }
\end{remark}
\begin{remark}
    Theorem \ref{fin_th_n} implies the convergence in law of $z^{(L+1)}(\mathcal{X})$ to $G^{(L+1)}(\mathcal{X})$ as the inner width $n$ diverge thanks to Proposition C.3.1 in \cite{NP12}.
\end{remark}
\begin{remark}
    Theorem 5 in \cite{CP25} is actually more general than Theorem \ref{fin_th_n}, since it also covers the gradients of the neural network with respect to its input. We do not include the gradients here to keep the notation simpler, but the results of Theorem \ref{edg_nn} can be extended to them under the same assumptions as in \cite{CP25}.
\end{remark}

Thanks to Lemma \ref{gaus_stut_NN}, the neural network $z^{(L+1)}(\mathcal{X})$ is conditionally Gaussian with conditional covariance matrix $(A^{(L+1)})^{\oplus n_{L+1}}$ (defined in \eqref{cov_z}) and hence we can apply Theorem \ref{edg_TV_generic} and, together with the study on the moments done in \cite{Han_Gas} (see Remark \ref{upp_bound_var_nn}), we can improve the quantitative Central Limit Theorem \ref{fin_th_n} including the non-Gaussian perturbations to the infinite-width limit. The proof of the following Theorem is presented more in detail in Section \ref{proof_th_edg_TV}.
\begin{theorem}\label{edg_nn}
Under the assumptions and notations of Theorem \ref{fin_th_n}, defining $Z:=z^{(L+1)}(\mathcal{X})$ with law $P_Z$, $G:=G^{(L+1)}(\mathcal{X})$, $K:=K^{(L+1)}$, $A:=A^{(L+1)}$ and recalling the definition of the measure $\gamma_{Z,G,m}$ in \eqref{edg_measure} for $m\ge 1$, we obtain that
{

\begin{equation}\label{bound_dTV_nn_no_exp}
   d_{TV}(P_Z,\gamma_{Z,G,m})\le \frac{C_1}{n^m},
\end{equation}
}
where $C_1\ge 0$ is a constant independent of the inner width $n$.

In the case when $\mathbb{E}[A]=K$, $L\ge 1$ and $\sigma$ is not constant, then, for every $i\in\{1,\dots,d\}$, we also obtain that 
\begin{equation}\label{lower_bound_egd_nn}
 d_{TV}(P_Z,\gamma_{Z,G,m})\ge \frac{C_2}{n^m},
\end{equation}
where $C_2> 0$ is a constant independent of the width $n$.
\end{theorem}
\begin{remark}
{
    Taking $m>1$, from the upper bound \eqref{bound_dTV_nn_no_exp} we observe that we obtain an improvement of Theorem \ref{fin_th_n}, where the neural network at initialization is approximated by its infinite-width Gaussian limit.
    }
\end{remark}
\begin{remark}
   The choice of considering the Edgeworth expansion of order $4m-1$ is inspired by the results in \cite{MPS25}. In that work, the authors study the total variation distance between the law of a real-valued random variable belonging to a Wiener chaos (see also \cite{NP12}), with unit variance, and the signed measure whose Radon--Nikodym density is given by its Edgeworth expansion with respect to a standard Gaussian random variable, at arbitrary order $4m-1$ for $m \in \mathbb{N}$.
As observed in Lemma \ref{gaus_stut_NN}, conditionally on the weights up to the layer $L$, the output of the neural network evaluated at a fixed input $x\in\mathbb{R}^{n_0}$ and taking real values (i.e. assuming $n_{L+1}=1$) is Gaussian, and therefore, conditionally, it belongs to the first Wiener chaos. As a consequence, by applying the results of \cite{MPS25} conditionally, together with the cumulant bounds established in \cite{Han_Gas}, one obtains an upper bound on the total variation distance between the law of the normalized network output
\[
\big(\mathbb{E}[A^{(L+1)}]\big)^{-1/2} \, z_1^{(L+1)}(x)
\]
and its corresponding Edgeworth expansion, defined as in \cite{MPS25}, of order $(1/\sqrt{n})^{m+1}$.
Therefore, Theorem~\ref{edg_nn} provides an improvement over the general results of \cite{MPS25} in this specific neural network setting whenever $m>1$. Note that, although the Edgeworth expansion in \cite{MPS25} appears different from the one in \eqref{edg_measure} (after the appropriate substitutions of $K$, $A$, and $n$), they are in fact equivalent, thanks to the result \eqref{eq_2_lemma_int_her} for the Hermite polynomials.

\end{remark}
\begin{remark}
   The proof of Theorem \ref{edg_nn} relies on three main ingredients: the conditional Gaussianity of the neural network (Lemma \ref{gaus_stut_NN}), Theorem \ref{edg_TV_generic}, and the estimates of \cite{Han_Gas} (see Remark \ref{upp_bound_var_nn} and bound \eqref{upp_bound_diff_exp}). 
Using the recent results of \cite{C26}, this framework can be extended beyond the fully Gaussian setting. In particular, it is sufficient to assume that the biases $(b_{i}^{(\ell)})_{i,\ell}$ satisfy Assumption \ref{bW}, that
\[
W_{i,j}^{(L+1)} \sim \mathcal{N}(0,1)
\quad \text{for all } i=1,\dots,n_{L+1},\; j=1,\dots,n_L,
\]
and that the weights
\[
(W_{i,j}^{(\ell)})_{i=1,\dots,n_\ell;\; j=1,\dots,n_{\ell-1}}^{\ell=1,\dots,L}
\]
are independent and identically distributed (not necessarily Gaussian) with finite moments of all orders (see \cite{C26} for more general assumptions).
 Moreover, assume that the activation function $\sigma$ is Lipschitz continuous and that the limiting covariance matrices $(K^{(\ell)})_\ell$ are invertible. Then, combining Lemma 10, Lemma 4, and Remark 14 of \cite{C26} with Theorem \ref{edg_TV_generic}, the total variation distance between the neural network output and the corresponding Edgeworth expansion \eqref{edg_measure} (after adequate substitutions of $A,K,n$) is of order
\[
O\!\left(\frac{1}{{n}^m}\right).
\]

To obtain a matching lower bound in total variation, assume that
\begin{equation}\label{cond_var_ge0}
\operatorname{Var}\!\left(\sigma\!\left(G_1^{(L)}(x^{(j)})\right)\right) > 0,
\end{equation}
where $G^{(L)}$ denotes the infinite-width Gaussian limit with covariance matrix defined as in \cite{Han22} (note that the assumptions of Theorem \ref{edg_nn} in this case are not enough to have \eqref{cond_var_ge0} since from \cite{Han22} it follows that $G^{(1)}(x^{(j)})=z^{(1)}(x^{(j)})$, which in general is not a Gaussian random variable if the weights are not Gaussians.). Then, using Theorem \ref{edg_TV_generic} together with Lemma 8, Lemma 7, and Remark 14 of \cite{C26}, there exists a constant $C>0$, independent of the inner width $n$, such that for $n$ sufficiently large, the total variation distance between the neural network and its Edgeworth expansion, taken with respect to the centered Gaussian vector whose covariance equals the expectation of the conditional covariance matrix of the network, is lower bounded by
\[
 \frac{C}{n^m},
\]
for $C>0$ constant independent of the inner width $n$.
{
We remark that the results of~\cite{C26}, and in particular Section~6.6 therein 
devoted to Gaussian initialization, could in principle be combined with 
Theorem~\ref{edg_TV_generic} to obtain a counterpart of our analysis in the 
regime where the depth $L$ and the width $n$ are simultaneously allowed to diverge. 
We do not pursue this extension here, since the analysis would considerably lengthen the paper and we leave it as a direction for future research.
}
\end{remark}

\begin{remark}\label{remark_picture}
{
Taking Figure 1 in \cite{NO24} as a reference, we compare the approximate density of a shallow real neural network (under the assumptions of Example \ref{es_NN_edg}) with four densities: the infinite-width Gaussian limit \eqref{gas_dens_es}, the first Edgeworth expansion \eqref{para_gauss_1_es}, an intermediate approximation \eqref{para_gauss_2_es}, and the second Edgeworth expansion \eqref{para_gauss_3_es}. As shown in Figure \ref{fig:edg_1dim}, when the inner width is small, the first two Edgeworth-based approximations mostly outperform the infinite-width Gaussian limit. For very large inner widths, however, the second Edgeworth expansion \eqref{para_gauss_3_es} yields in general the best approximation. The deterioration of this latter expansion at small widths can be attributed to the oscillatory nature of the higher-order Hermite polynomials appearing in the series, whose contributions become dominant before the expansion has had a chance to converge. This behavior is consistent with the well-known limitations of Edgeworth expansions at moderate sample sizes (see e.g.\ \cite{H92}).
}
\end{remark}
\begin{remark}
Theorem \ref{edg_nn} shows that the rate $n^{-m}$ is {optimal}: when the Edgeworth expansion is done with respect to a Gaussian vector with covariance given by the expectation of the conditional covariance matrix of the neural network, we prove a matching lower bound of order $n^{-m}$.

To our knowledge, this is the first work establishing quantitative and high-order approximation bounds for multidimensional neural network outputs via Edgeworth expansions, together with matching lower bounds.  
\end{remark}

\begin{figure}[ht] 
    \centering
    \includegraphics[width=0.9\textwidth]{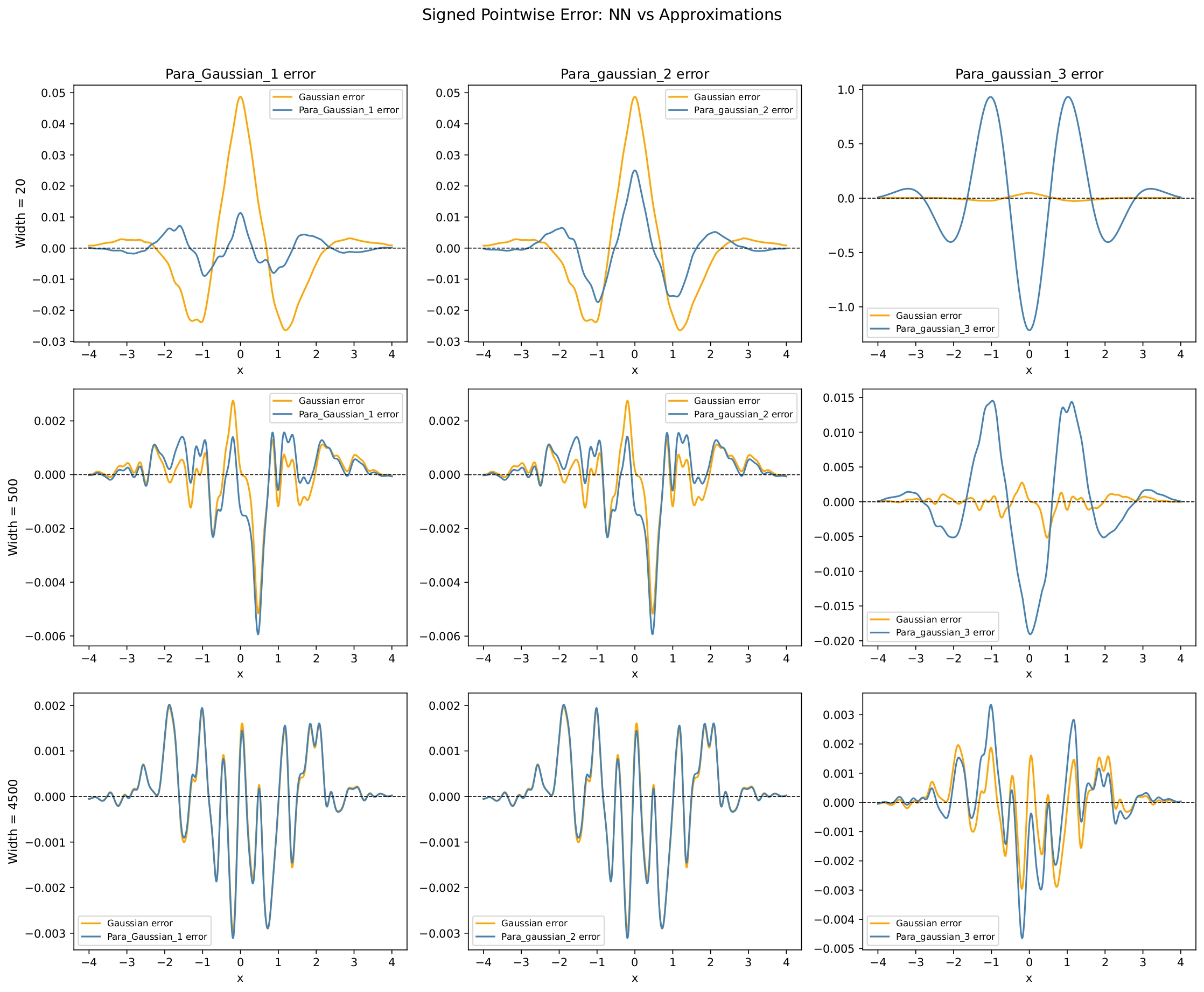} 
    \caption{
    { The figure displays the signed pointwise error between the Monte Carlo 
kernel density estimate~(KDE) of the output distribution of a shallow 
neural network, as defined in Example~\ref{es_NN_edg}, denoted 
$\rho_{\text{NN}}$, and each of the following approximations: the 
Gaussian approximation $\rho_{\text{Gaussian}}$ given 
by~\eqref{gas_dens_es}, and the higher-order Edgeworth (Para-Gaussian) 
approximations \texttt{Para\_Gaussian\_1}, \texttt{Para\_Gaussian\_2}, 
and \texttt{Para\_Gaussian\_3}, given respectively 
by~\eqref{para_gauss_1_es}, \eqref{para_gauss_2_es}, 
and~\eqref{para_gauss_3_es}. Results are shown for inner widths $n = 20$, 
$n = 500$, and $n = 4500$.
More precisely, in each panel the orange curve represents the error 
$\rho_{\text{NN}} - \rho_{\text{Gaussian}}$, while the blue curve 
represents $\rho_{\text{NN}} - \texttt{Para\_Gaussian\_1}$ in the first 
column, $\rho_{\text{NN}} - \texttt{Para\_Gaussian\_2}$ in the second 
column, and $\rho_{\text{NN}} - \texttt{Para\_Gaussian\_3}$ in the third 
column. As noted in Remark~\ref{remark_picture}, the poor behaviour of the 
third-order approximation at small inner widths is expected, and is due 
to the nature of the Edgeworth expansion itself.}}
    \label{fig:edg_1dim} 
\end{figure}

\section{Bayesian supervised learning with neural networks}\label{sec:bayesian}

Consider a supervised learning problem with training dataset
\begin{equation}\label{data}
\mathcal{D} := \{(x^{(i)}, y^{(i)})\}_{i=1}^d \subseteq \left(\mathbb{R}^{n_0}\setminus\{0\}\right) \times \mathbb{R}^{n_{L+1}}.
\end{equation}
Assume that the labels are generated by a continuous target function evaluated at the corresponding inputs. The aim is to approximate this function using a neural network with parameters
\begin{equation}\label{par_theta}
\Theta := \{ b_i^{(\ell)},\, W_{i,j}^{(\ell)} \}_{i,j,\ell},
\end{equation}
see, for instance,~\cite{Cybenko}.

We adopt a Bayesian perspective on this problem (see, e.g.,~\cite{GPML,Trev,Fort22,HBNPS,CP25}). Following \cite{HBNPS,Trev,CP25}, we assume that the likelihood depends on the parameters only through the network output, namely it is of the form
\[
\mathcal{L}\bigl(z^{(L+1)}(\mathcal{X}\mid\Theta)\bigr),
\]
where $\mathcal{L} : \mathbb{R}^{ n_{L+1} d} \to \mathbb{R}_+$ is a non-negative continuous function, $
\mathcal{X} := \{x^{(1)}, \dots, x^{(d)}\}$ and we write $z^{(L+1)}(\mathcal{X}\mid\Theta)$ to stress the dependence of the neural network on the parameters $\Theta$.

A prior distribution is specified on the parameter vector $\Theta$, which induces a prior probability measure on the neural network, viewed as a random variable taking values in an appropriate functional space.
By Bayes' Theorem, the posterior measure $\mu_{\mid \mathcal{D}}$ is then defined as
\begin{equation}\label{post_nn}
\mathrm{d}\mu_{\mid \mathcal{D}}(z)
:= \frac{\mathcal{L}(z)}{\mathbb{E}\bigl[\mathcal{L}\bigl(z^{(L+1)}(\mathcal{X}\mid\Theta_0)\bigr)\bigr]}\,
\mathrm{d}\mu(z),
\end{equation}
where $\Theta_0$ denotes the random parameter vector at initialization, distributed according to the prior.

The posterior is then used to make predictions for the value of the unknown function at a new input $x^* \in \mathbb{R}^{n_0}$. This defines a predictive probability measure on $\mathbb{R}^{n_{L+1}}$ given by
\begin{equation}\label{pred_distr_nn}
B \in \mathcal{B}(\mathbb{R}^{n_{L+1}}) \longmapsto\frac{1}{\mathbb{E}\bigl[\mathcal{L}\bigl(z^{(L+1)}(\mathcal{X}\mid\Theta_0)\bigr)\bigr]}
\mathbb{E}\Bigl[
\mathbf{1}_B\bigl(z^{(L+1)}(x^* \mid \Theta_0)\bigr)\,
\mathcal{L}\bigl(z^{(L+1)}(\mathcal{X} \mid \Theta_0)\bigr)
\Bigr],
\end{equation}
where $\mathbf{1}_B$ denotes the indicator function of the Borel set $B$.

Recently, in~\cite{HBNPS,Trev,CP25}, the problem of the convergence in law of the posterior distribution of a fully connected neural network under Gaussian initialization to that of its Gaussian limit (defined in Theorem~\ref{fin_th_n}) has been studied, in the regime where the hidden-layer widths diverge to infinity. In particular,~\cite{Trev,CP25} established a convergence rate of order $1/\mathrm{width}$ for several probability distances, under suitable assumptions on the likelihood function. 

In the following Theorem, we show that using as a prior the Edgeworth expansion defined in~\eqref{edg_measure} (with $Z,G,K,A$ as in Theorem \ref{edg_nn}) instead of the Gaussian infinite-width limit yields an improved convergence rate and therefore a more accurate approximation of the posterior laws in total variation distance.

\begin{theorem}\label{bayes_th_edg}
    Assume the condition on the inner widths \eqref{bounds_n} and suppose that $n>>1$. 
    Define $\mu$ as the law of the neural network at initialization $z^{(L+1)}(\mathcal{X}\mid\Theta_0)$, where $\Theta_0$ is given by \eqref{par_theta} under Assumption \ref{bW}. Define the corresponding posterior measure $\mu_{|\mathcal{D}}$ as in \eqref{post_nn} and define 
    \begin{equation}\label{approx_post_distr}
    \gamma_{|\mathcal{D}}:=\frac{\mathcal{L}(x)}{\int_{\mathbb{R}^{ n_{L+1} d}}\mathcal{L}(x)d\gamma_{Z,G,m}(x)}d\gamma_{Z,G,m}(x),
    \end{equation}
    where  $\mathcal{L} : \mathbb{R}^{ n_{L+1} d} \to \mathbb{R}_+$ is a non-negative bounded continuous function and $\gamma_{Z,G,m}$ is defined in \eqref{edg_measure} with $Z,G,K,A$ as in Theorem \ref{edg_nn}.
    If $\mathbb{E}\left[\mathcal{L}\left(G^{(L+1)}(\mathcal{X}\right)\right]>0$ and the matrices defined in \eqref{lim_cov}, $\{K^{(\ell)}\}_{\ell=1,\dots,L+1}$, are invertible on $\mathcal{X}$,
    then there exists a constant $D\ge 0$ independent of $n$ such that
    {
    \[
    d_{TV}(\mu_{|\mathcal{D}},\gamma_{|\mathcal{D}})\le \frac{D}{n^m}.
    \]
    }
  
\end{theorem}
\begin{proof}
The argument follows the same strategy as the proof of Theorem~6 in~\cite{CP25}, combined with Theorem~\ref{edg_nn} and with suitable bounds on the measures $\gamma_{Z,G,m}$, already established in the proof of Theorem~\ref{edg_TV_generic} (see Section~\ref{proof_th_edg_TV_gen}).

We also rely on the following observation: if
\[
\mathbb{E}\left[\mathcal{L}\left(G^{(L+1)}(\mathcal{X})\right)\right] > 0,
\]
then
\[
\left| \int_{\mathbb{R}^{n_{L+1}d}} \mathcal{L}(x)\, \mathrm{d}\gamma_{Z,G,m}(x) \right| > 0.
\]
This conclusion follows from the bounding techniques developed in the proof of Theorem~\ref{edg_TV_generic}, together with the reverse triangle inequality for the absolute value.
\end{proof}

\subsection{{The case of the Gaussian Likelihood function}}
We now show that, in the case of a Gaussian likelihood, the posterior distribution $\gamma_{|\mathcal{D}}$, defined in \eqref{approx_post_distr} and under the same assumptions of Theorem \ref{bayes_th_edg}, admits an explicit expression.

For simplicity, assume $n_{L+1}=1$ and consider a Gaussian likelihood of the form
\[
\mathcal{L}\bigl(z^{(L+1)}(\mathcal{X}\mid\Theta)\bigr),
\qquad
\mathcal{L}(z):=\exp\!\left(-\frac{1}{2}\sum_{i=1}^d |z_i-y^{(i)}|^2\right),
\quad z\in\mathbb{R}^d,
\]
recalling that $\{y^{(i)}\}_{i=1,\dots,d}\subseteq \mathbb{R}$ denote the labels (see \eqref{data}).
Then, following the approach of~\cite{GPML}, the approximating posterior distribution $\gamma_{|\mathcal{D}}$ defined in \eqref{approx_post_distr} admits the explicit expression
\begin{multline}\label{approx_post_distr_gaus}
z=(z_j)_{j=1,\dots,d}\longmapsto
\frac{1}{\mathcal{T}}\Bigg(1+\sum_{k=1}^{2m-1}\frac{1}{k!\,2^k}
\sum_{J\in S_d^{(2k)}}\sum_{\alpha\in\mathcal{A}_J}
\mathbb{E}\!\left[Q_{\alpha_1,\alpha_2}\cdots Q_{\alpha_{2k-1},\alpha_{2k}}\right]\\
\cdot
\left(H_{j_1}(u_1)\cdots H_{j_d}(u_d)\right)\Big|_{u=(\sqrt{K^{(L+1)}})^{-1}z}
\Bigg)
\exp\!\left(-\frac{1}{2}\bigl\langle z-\mu,\Sigma(z-\mu)\bigr\rangle\right),
\end{multline}
where $\mathcal{T}$ is a normalizing constant ensuring that the total mass is one and $\mu$ and $\Sigma$ are given by
\[
\mu=\bigl((K^{(L+1)})^{-1}+I_d\bigr)^{-1}Y,
\qquad
\Sigma=(K^{(L+1)})^{-1}+I_d,
\]
where $Y=(y^{(1)},\dots,y^{(d)})\in\mathbb{R}^d$ and $I_d\in\mathbb{R}^{d\times d}$ denotes the identity matrix.

In particular, the normalizing constant $\mathcal{T}$ can be written explicitly as
\begin{multline}\label{norm_const_bayes}
\mathcal{T}=\frac{1}{(2\pi)^{d/2}\sqrt{\det(\Sigma)}}\Bigg(1+\sum_{k=1}^{2m-1}\frac{1}{k!\,2^k}
\sum_{J\in S_d^{(2k)}}\sum_{\alpha\in\mathcal{A}_J}
\mathbb{E}\!\left[Q_{\alpha_1,\alpha_2}\cdots Q_{\alpha_{2k-1},\alpha_{2k}}\right]\\
\cdot
\sum_{r_1=j_1-2\lfloor j_1/2\rfloor}^{j_1}\cdots
\sum_{r_d=j_d-2\lfloor j_d/2\rfloor}^{j_d}1_{\{r_1+\dots+r_d=\text{even}\}}
\binom{j_1}{r_1}\cdots\binom{j_d}{r_d}
\bigl((\sqrt{K}^{-1}\mu)_1\bigr)^{r_1}\cdots
\bigl((\sqrt{K}^{-1}\mu)_d\bigr)^{r_d}\\
\cdot
\dfrac{(j_1-r_1)!\dots(j_d-r_d)!}{2^{(r_1+\dots+r_d)/2}((r_1+\dots+r_d)/2)!}\displaystyle\sum_{\alpha\in\mathcal{A}_{r_1+\dots+r_d}}
R_{\alpha_1,\alpha_2}\cdots R_{\alpha_{r_1+\dots+r_d-1},\alpha_{r_1+\dots+r_d}}
\Bigg),
\end{multline}
where
\[
R:=(I_d+K^{(L+1)})^{-1}-I_d.
\]

Identity~\eqref{norm_const_bayes} immediately follows from the next lemma, whose proof is given in Appendix~\ref{proof_norm_const}.

\begin{lemma}\label{prop_herm_with_expectation}
Let $J=(j_1,\dots,j_d)\in S_d^{(2k)}$. Then
\begin{multline}\label{eq_1_lemma_int_her}
\mathbb{E}\!\left[
H_{j_1}\!\left((\sqrt{K}^{-1}\tilde{G})_1\right)\cdots
H_{j_d}\!\left((\sqrt{K}^{-1}\tilde{G})_d\right)
\right]\\
=\sum_{r_1=j_1-2\lfloor j_1/2\rfloor}^{j_1}\cdots
\sum_{r_d=j_d-2\lfloor j_d/2\rfloor}^{j_d}
\binom{j_1}{r_1}\cdots\binom{j_d}{r_d}
\bigl((\sqrt{K}^{-1}\mu)_1\bigr)^{r_1}\cdots
\bigl((\sqrt{K}^{-1}\mu)_d\bigr)^{r_d}\\
\cdot
\mathbb{E}\!\left[H_{j_1-r_1}(\tilde{Z}_1)\cdots H_{j_d-r_d}(\tilde{Z}_d)\right],
\end{multline}
where $\tilde{G}\sim\mathcal{N}_d(\mu,\Sigma)$ and
$\tilde{Z}\sim\mathcal{N}_d\!\left(0,\sqrt{K}^{-1}\Sigma\sqrt{K}^{-1}\right)$.
Moreover, for any $s=(s_1,\dots,s_d)\in S_d^{(v)}$,
\begin{equation}\label{eq_2_lemma_int_her}
\mathbb{E}\!\left[H_{s_1}(\tilde{Z}_1)\cdots H_{s_d}(\tilde{Z}_d)\right]=
\begin{cases}
\dfrac{s_1!\dots s_d!}{2^{v/2}(v/2)!}\displaystyle\sum_{\alpha\in\mathcal{A}_s}
R_{\alpha_1,\alpha_2}\cdots R_{\alpha_{v-1},\alpha_v},
& \text{if $v$ is even},\\[1ex]
0, & \text{if $v$ is odd},
\end{cases}
\end{equation}
with $R=\sqrt{K}^{-1}\Sigma\sqrt{K}^{-1}-I_d$.
\end{lemma}
{
\begin{example}
Recalling the results from Example \ref{ex_edg_dens} for $n=1$,  
\[
S_{d}^{(2)}=\left\{\eta\left(2,0,\dots,0\right)\right\}_{\eta\in\Pi_{d}}\cup\left\{\eta\left(1,1,0\dots,0\right)\right\}_{\eta\in\Pi_{d}},
\] 
denoting with $\Pi_{d}$ the sets of all the permutations of $d$ elements and writing 

\noindent
$\eta(a_1,\dots,a_{d}):=(a_{\eta(1)},\dots a_{\eta(d)})$ for every vector $(a_1,\dots,a_{d})\in\mathbb{R}^{d}$ and $\eta \in \Pi_{d}$.

Hence, taking $\eta\in\Pi_d$ and
\[
J=\eta\left(2,0,\dots,0\right)\in S_{d}^{(2)}
\]
in Lemma \ref{prop_herm_with_expectation}, we obtain, keeping the notations of the Lemma, that 
\begin{multline*}
\mathbb{E}\!\left[
H_{j_1}\!\left((\sqrt{K}^{-1}\tilde{G})_1\right)\cdots
H_{j_d}\!\left((\sqrt{K}^{-1}\tilde{G})_d\right)
\right]
=\sum_{r=0}^{2}
\binom{2}{r}
\bigl((\sqrt{K}^{-1}\mu)_{\eta(1)}\bigr)^{r}
\mathbb{E}\!\left[H_{2-r}(\tilde{Z}_{\eta(1)})\right]\\
=\mathbb{E}\!\left[H_{2}(\tilde{Z}_{\eta(1)})\right]+
\bigl((\sqrt{K}^{-1}\mu)_{\eta(1)}\bigr)^{2}
=\displaystyle\sum_{\alpha\in\mathcal{A}_J} R_{\alpha_1,\alpha_2}+\bigl((\sqrt{K}^{-1}\mu)_{\eta(1)}\bigr)^{2}\\
=R_{\eta(1),\eta(1)}+\bigl((\sqrt{K}^{-1}\mu)_{\eta(1)}\bigr)^{2},
\end{multline*}
since in this case $\mathcal{A}_J=\{(\eta(1),\eta(1))\}$.

If instead we consider $\bar{J}=\bar{\eta}(1,1,0,\dots,0)\in S^{(2)}_d$ for a certain $\bar{\eta}\in\Pi_d$, then
\begin{multline*}
\mathbb{E}\!\left[
H_{\bar{j}_1}\!\left((\sqrt{K}^{-1}\tilde{G})_1\right)\cdots
H_{\bar{j}_d}\!\left((\sqrt{K}^{-1}\tilde{G})_d\right)
\right]\\
=\sum_{r_1=0}^{1}
\sum_{r_2=0}^{1}
\binom{1}{r_1}\binom{1}{r_2}
\bigl((\sqrt{K}^{-1}\mu)_{\bar{\eta}(1)}\bigr)^{r_1}
\bigl((\sqrt{K}^{-1}\mu)_{\bar{\eta}(2)}\bigr)^{r_2}
\mathbb{E}\!\left[H_{1-r_1}(\tilde{Z}_{\bar{\eta}(1)}) H_{1-r_2}(\tilde{Z}_{\bar{\eta}(2)})\right]\\
=\mathbb{E}\!\left[H_{1}(\tilde{Z}_{\bar{\eta}(1)}) H_{1}(\tilde{Z}_{\bar{\eta}(2)})\right]+(\sqrt{K}^{-1}\mu)_{\bar{\eta}(2)}\mathbb{E}\!\left[H_{1}(\tilde{Z}_{\bar{\eta}(1)})\right]+(\sqrt{K}^{-1}\mu)_{\bar{\eta}(1)}\mathbb{E}\!\left[H_{1}(\tilde{Z}_{\bar{\eta}(2)})\right]\\
+(\sqrt{K}^{-1}\mu)_{\bar{\eta}(1)}
(\sqrt{K}^{-1}\mu)_{\bar{\eta}(2)}
=R_{\bar{\eta}(1),\bar{\eta}(2)}
+(\sqrt{K}^{-1}\mu)_{\bar{\eta}(1)}
(\sqrt{K}^{-1}\mu)_{\bar{\eta}(2)},
\end{multline*}
using that $H_1(x)=x$.
\end{example}

\begin{example}
    In the particular case of $n=d=1$, then, as seen in Example \ref{es_edg_11}, we have that for every $k\in\mathbb{N}$
    \[
    S_1^{(2k)}=\{(2k)\}\quad\text{and}\quad \mathcal{A}_J=\{(1,\dots,1)\}\subseteq \mathbb{N}^{2k}\quad\text{for $J\in S_1^{(2k)}$}.
    \]
    Therefore, for $J=(2k)$,
    \begin{multline*}  \mathbb{E}\!\left[H_{2k}\!\left((\sqrt{K}^{-1}\tilde{G})\right)\right]=\sum_{r=0}^{2k}\binom{2k}{r}\bigl(\sqrt{K}^{-1}\mu\bigr)^{r}\mathbb{E}\!\left[H_{2k-r}(\tilde{Z})\right]\\
    =\sum_{r=0}^{k}\binom{2k}{2r}\bigl(\sqrt{K}^{-1}\mu\bigr)^{2r}\mathbb{E}\!\left[H_{2k-2r}(\tilde{Z})\right]=\sum_{r=0}^{k}\binom{2k}{2r}\frac{(2k-2r)!}{2^{k-r}(k-r)!}(R_{1,1})^{k-r}\bigl(\sqrt{K}^{-1}\mu\bigr)^{2r},
    \end{multline*}
    recalling that in this case $R=R_{1,1}\in\mathbb{R}$.
\end{example}
}

\appendix
\section{Proof of Theorem \ref{edg_TV_generic}}\label{proof_th_edg_TV_gen}

We consider a random vector $Z$ with values in $\mathbb{R}^{n d}$ and with Gaussian distribution conditionally on a $\sigma$-field $\mathcal{F}$, zero expectation and conditional covariance matrix given by $A^{\oplus n}$ (see Definition \ref{cond_Ga_def}), which denotes a block matrix with $n$ diagonal blocks all equal to a matrix $A\in\mathbb{R}^{d\times d}$. We denote the law of $Z$ by $P_Z$.

We also define $G\sim\mathcal{N}_{n d}(0,K^{\oplus n})$ with $K$ invertible and assume, without loosing of generality, that $G$ is independent of $A$.

In the remainder of the paper, for any positive definite matrix $C$ we will write $\phi_C$ to denote the density of a Gaussian vector with zero expectation and the covariance matrix given by $C$.

\subsection{Proof of the upper bound \eqref{bound_dTV_gen_no_exp} on the Total Variation distance }
Fix a function $h:\mathbb{R}^{n d}\to\mathbb{R}$ with $|h(x)|\le 1$ for every $x\in\mathbb{R}^{n d}$ as in Definition \ref{dTV_def}. We want to study
\[
\left|\int_{\mathbb{R}^{n d}}h(x)dP_Z(x)-\int_{\mathbb{R}^{n d}}h(x)d\gamma_{Z,G,m}(x)\right|,
\]
where $\gamma_{Z,G,m}$ is the measure with density defined in \eqref{edg_measure}. To do this, following the idea in \cite{Trev}, we introduce the event (recalling that $A$ is a random matrix)
\begin{equation}\label{def_ev_E}
E:=\left\{\|A-K\|_{op}\le\frac{\lambda(K)}{2}\right\},
\end{equation}
where $\|\cdot\|_{op}$ denotes the operator norm of a matrix and $\lambda(K)$ denotes the minimum eigenvalue of $K$, which is invertible by assumption.

\begin{remark}\label{inv_A_in_E}
    As observed in \cite{Trev,CP25}, if $\|A-K\|_{\text{op}}\le \frac{\lambda(K)}{2}$ then $A$ is invertible because for every $x\in\mathbb{R}^d$ with $\|x\|=1$ one has that
\[
x^TAx=x^T(A-K)x+x^TKx\ge \lambda(K)-\|A-K\|_{\text{op}}\ge \frac{\lambda(K)}{2}>0.
\]
\end{remark}
Thanks to Remark \ref{inv_A_in_E} and Lemma \ref{gaus_stut_NN}, in the event $E$ defined in \ref{def_ev_E} we have that, conditioning on the $\sigma$-field $\mathcal{F}$, the random vector $Z$ has a density 
\[
x\in\mathbb{R}^{nd}\mapsto \prod_{i=1}^n\phi_A(\tilde{x}_{i}),
\]
where $\tilde{x}_i=\left(x_{d(i-1)+1},\dots,x_{d(i-1)+d}\right)$ for every $i=1,\dots,n$.
Hence 
\begin{multline}\label{divido_E}
\left|\int_{\mathbb{R}^{n d}}h(x)dP_Z(x)-\int_{\mathbb{R}^{n d}}h(x)d\gamma_{Z,G,m}(x)\right|\\
\le \left|\mathbb{E}\left[\int_{\mathbb{R}^{n d}}h(x)1_E\prod_{i=1}^n\phi_A(\tilde{x}_{i})dx-\int_{\mathbb{R}^{nd}}h(x)d\gamma_{Z,G,m}(x)\right]\right|
+\left|\mathbb{E}\left[h(Z)1_{E^C}\right]\right|.
\end{multline}

Writing explicitly the definition of $\gamma_{Z,G,m}$ we obtain that
\begin{multline*}
\left|\mathbb{E}\left[\int_{\mathbb{R}^{n d}}h(x)1_E\prod_{i=1}^n\phi_A(\tilde{x}_{i})dx-\int_{\mathbb{R}^{nd}}h(x)d\gamma_{Z,G,m}(x)\right]\right|
=\Bigg|\mathbb{E}\Bigg[\int_{\mathbb{R}^{n d}}h(x)1_E\prod_{i=1}^n\phi_A(\tilde{x}_{i})dx\\
-\int_{\mathbb{R}^{n d}}h(x)\Bigg(
\prod_{i=1}^{n}\phi_K(\tilde{x}_{i})+\sum_{k=1}^{2m-1}\frac{1}{k!2^k}\sum_{J\in S_{n d}^{(2k)}}\sum_{\alpha\in \mathcal{A}_J}\mathbb{E}\left[(Q^{\oplus n})_{\alpha_1,\alpha_2}\dots(Q^{\oplus n})_{\alpha_{2k-1},\alpha_{2k}}\right]\cdot\\
    \cdot \prod_{i=0}^{n-1} \left(H_{j_{di+1}}\left(u^{(i)}_1\right)\dots H_{j_{di+d}}\left(u^{(i)}_{d}\right)\right)_{|_{u^{(i)}=\sqrt{K}^{-1}\tilde{x}_{i+1}}}\phi_K\left(\tilde{x}_{i+1}\right)
\Bigg)dx\Bigg]\Bigg|,
\end{multline*}
where $Q:=\sqrt{K}^{-1}(A-K)\sqrt{K}^{-1}$, $H_j$ for $j\in\mathbb{N}\cup\{0\}$ are the Hermite polynomials defined in \eqref{herm_def} and $S_{n d}^{(k)}$ and $\mathcal{A}_J$ are defined respectively in \eqref{def_S_k} and \eqref{def_A_j}.

Dividing the expectations in the definition of $\gamma_{Z,G,m}$ in the event $E$ and in its complementary $E^C$, we obtain that
\begin{multline}\label{eq_dopo_lemma_coeff_edg}
\left|\mathbb{E}\left[\int_{\mathbb{R}^{n d}}h(x)1_E\prod_{i=1}^n\phi_A(\tilde{x}_{i})dx-\int_{\mathbb{R}^{nd}}h(x)d\gamma_{Z,G,m}(x)\right]\right| 
   \le \Bigg|\mathbb{E}\Bigg[\int_{\mathbb{R}^{n d}}h(x)1_E\prod_{i=1}^n\phi_A(\tilde{x}_{i})dx\\
-\int_{\mathbb{R}^{n d}}h(x) 1_E\Bigg(
\prod_{i=1}^{n}\phi_K(\tilde{x}_{i})+\sum_{k=1}^{2m-1}\frac{1}{k!2^k}\sum_{J\in S_{n d}^{(2k)}}\sum_{\alpha\in \mathcal{A}_J}(Q^{\oplus n})_{\alpha_1,\alpha_2}\dots(Q^{\oplus n})_{\alpha_{2k-1},\alpha_{2k}}\cdot\\
    \cdot \prod_{i=0}^{n-1} \left(H_{j_{di+1}}\left(u^{(i)}_1\right)\dots H_{j_{di+d}}\left(u^{(i)}_{d}\right)\right)_{|_{u^{(i)}=\sqrt{K}^{-1}\tilde{x}_{i+1}}}\phi_K\left(\tilde{x}_{i+1}\right)
\Bigg)dx\Bigg]\Bigg|\\
+\Bigg|\mathbb{E}\Bigg[\int_{\mathbb{R}^{n d}}h(x)1_{E^C}\Bigg(
\prod_{i=1}^{n}\phi_K(\tilde{x}_{i})+\sum_{k=1}^{2m-1}\frac{1}{k!2^k}\sum_{J\in S_{n d}^{(2k)}}\sum_{\alpha\in \mathcal{A}_J}(Q^{\oplus n})_{\alpha_1,\alpha_2}\dots(Q^{\oplus n})_{\alpha_{2k-1},\alpha_{2k}}\cdot\\
    \cdot \prod_{i=0}^{n-1} \left(H_{j_{di+1}}\left(u^{(i)}_1\right)\dots H_{j_{di+d}}\left(u^{(i)}_{d}\right)\right)_{|_{u^{(i)}=\sqrt{K}^{-1}\tilde{x}_{i+1}}}\phi_K\left(\tilde{x}_{i+1}\right)
\Bigg)dx\Bigg]\Bigg|.
\end{multline}

We now use the following Lemma, which can be extrapolated from the proof of Proposition 4 in \cite{CP25} and that for completeness has been proved in Subsection \ref{sec_proof_lemma_der_phi_K_in_0}.
\begin{lemma}\label{lemma_der_phi_K_in_0}
For any $k\in\mathbb{N}\cup\{0\}$ and $t\in[0,1]$, denote 
\[
\Gamma_t:=tA+(1-t)K\quad\text{and}\quad \tilde{Q}_t:=\sqrt{\Gamma_t}^{-1}(A-K)\sqrt{\Gamma_t}^{-1}
\]
and recall that $\phi_C$ denotes the density of a Gaussian vector with zero expectation and covariance matrix given by $C$. Then, in the event where the matrix $A$ is positive definite, one has that
\begin{multline*}
\frac{\partial^k}{\partial t^k}\phi_{\Gamma_t^{\oplus n}}(x)=\frac{1}{2^k}\sum_{J\in S_{n d}^{(2k)}}\sum_{\alpha\in \mathcal{A}_J}(\tilde{Q}_t^{\oplus n})_{\alpha_1,\alpha_2}\dots (\tilde{Q}_t^{\oplus n})_{\alpha_{2k-1},\alpha_{2k}}\cdot\\
\cdot\prod_{i=0}^{n-1}\left(H_{j_{di+1}}(u^{(i)}_1)\dots H_{j_{di+d}}(u^{(i)}_d)\right)_{|_{u^{(i)}=\sqrt{\Gamma_t}^{-1}\tilde{x}_{i+1}}}\phi_{\Gamma_t^{\oplus n}}(x)
\end{multline*}
for any $x\in\mathbb{R}^{nd}$ and denoting $\tilde{x}_i=\left(x_{d(i-1)+1},\dots,x_{d(i-1)+d}\right)$ for every $i=1,\dots,n$.
\end{lemma}
Applying Lemma \ref{lemma_der_phi_K_in_0} to \eqref{eq_dopo_lemma_coeff_edg}, defining $\Gamma_t$ as in Lemma \ref{lemma_der_phi_K_in_0} and writing  $G\stackrel{Law}{=}(\sqrt{K})^{\oplus n}N$, where $N\sim\mathcal{N}_{n d}(0,I_{nd})$ is independent of $A$ and $I_{nd}$ denotes the identity matrix of dimension $nd\times nd$, we obtain that
\begin{multline}\label{prima_taylor}
 \Bigg|\mathbb{E}\left[\int_{\mathbb{R}^{n d}}h(x)1_E\prod_{i=1}^n\phi_A(\tilde{x}_{i})-\int_{\mathbb{R}^{n d}}h(x)d\gamma_{Z,G,m}(x)\right]\Bigg|\\
 \le\Bigg| \mathbb{E}\Bigg[\int_{\mathbb{R}^{n d}}h(x)1_E\Bigg(\prod_{i=1}^n\phi_A(\tilde{x}_{i})
-\prod_{i=1}^n\phi_K(\tilde{x}_{i})
-\sum_{k=1}^{2m-1}\frac{1}{k!}\frac{\partial^k}{\partial t^k}\phi_{\Gamma_t^{\oplus n}}(x)_{|_{t=0}}\Bigg)\Bigg]\Bigg|\\
+\Bigg|\mathbb{E}\Bigg[h\left((\sqrt{K})^{\oplus n}N\right)1_{E^C}\Bigg(1+\sum_{k=1}^{2m-1}\sum_{J\in S_{n d}^{(2k)}}\frac{1}{2^kk!}\sum_{\alpha\in \mathcal{A}_J}(Q^{\oplus n})_{\alpha_1,\alpha_1}\dots (Q^{\oplus n})_{\alpha_{2k-1},\alpha_{2k}}\cdot\\
\cdot \prod_{i=0}^{n-1}\left(H_{j_{di+1}}(N_{di+1})\dots H_{j_{di+d}}(N_{di+d})\right)\Bigg]\Bigg|.
\end{multline}

Doing a Taylor expansion of $t\mapsto \prod_{i=1}^n\phi_{\Gamma_t}(\tilde{x}_{i})=\phi_{\Gamma_t^{\oplus n}}(x)$ around $t=0$ in the event $E$, we have that for every $x\in\mathbb{R}^{n d}$
\begin{equation}\label{taylor_phi_t}
\prod_{i=1}^n\phi_A(\tilde{x}_{i})=\prod_{i=1}^n\phi_K(\tilde{x}_{i})+\sum_{k=1}^{2m-1}\frac{1}{k!}\frac{\partial ^k}{\partial t^k}\phi_{\Gamma_t^{\oplus n}}(x)_{|_{t=0}}+R_{2m-1}(x,A,K),
\end{equation}
where the remainder $R_{2m-1}(x,A,K)$ can be explicitly written (see e.g. \cite{Apostol67}) as
\begin{multline}\label{resto}
R_{2m-1}(x,A,K)=\frac{1}{(2m-1)!}\int_{0}^1(1-t)^{2m-1}\frac{\partial^{2m}}{\partial t^{2m}}\left(\phi_{\Gamma_t^{\oplus n}}(x)\right)dt\\
=\frac{1}{2^{2m}(2m-1)!}\int_0^1(1-t)^{2m-1}\sum_{J\in S_{nd}^{(4m)}}\sum_{\alpha\in \mathcal{A}_J}(\tilde{Q}_t^{\oplus n})_{\alpha_1,\alpha_2}\dots (\tilde{Q}_t^{\oplus n})_{\alpha_{4m-1},\alpha_{4m}}\cdot\\
\cdot\prod_{i=0}^{n-1}\left(H_{j_{di+1}}(u^{(i)}_1)\dots H_{j_{di+d}}(u^{(i)}_d)\right)_{|_{u^{(i)}=\sqrt{\Gamma_t}^{-1}\tilde{x}_{i+1}}}\phi_{\Gamma_t^{\oplus n}}(x)dt,
\end{multline}
using Lemma \ref{lemma_der_phi_K_in_0} in the last identity and calling $ \tilde{Q}_t:=\sqrt{\Gamma_t}^{-1}(A-K)\sqrt{\Gamma_t}^{-1}$.

From \eqref{divido_E},\eqref{prima_taylor} and \eqref{taylor_phi_t}, computing the supremum over $h\in\mathcal{M}_1$, it follows that
\begin{multline*}
d_{\text{TV}}\left(P_Z,\gamma_{Z,G,m}\right)\le \frac{1}{2}\sup_{\substack{h:\mathbb{R}^{nd}\to\mathbb{R}\\\|h\|_{\infty}\le 1}}\Bigg\{\left|\mathbb{E}\left[h(Z)1_{E^C}\right]\right|+\Bigg|\mathbb{E}\Bigg[\int_{\mathbb{R}^{n d}}h(x)1_ER_{2m-1}({x},A,K)dx\Bigg]\Bigg|\\
+\left|\mathbb{E}\left[h(G)1_{E^C}\right]\right|
+\Bigg|\sum_{k=1}^{2m-1}\sum_{J\in S_{n d}^{(2k)}}\frac{1}{2^kk!}\sum_{\alpha\in \mathcal{A}_J}\mathbb{E}\left[(Q^{\oplus n})_{\alpha_1,\alpha_2}\dots (Q^{\oplus n})_{\alpha_{2k-1},\alpha_{2k}}1_{E^C}\right]\cdot\\
\cdot \mathbb{E}\left[h\left((\sqrt{K})^{\oplus n}N\right)\prod_{i=0}^{n-1}H_{j_{di+1}}(N_{di+1})\dots H_{j_{di+d}}(N_{di+d})\right]\Bigg|\Bigg\}.
\end{multline*}

Using that $\mathbb{E}[H_p(N_i)^2]=p!$ for every $p\in\mathbb{N}$ and $i=1,\dots,d$ (see Proposition 1.4.2 in \cite{NP12}), Markov inequality and that $|h(x)|\le 1$ for every $x\in\mathbb{R}^{n d}$,
\begin{multline*}
d_{\text{TV}}\left(P_Z,\gamma_{Z,G,m}\right)
\le \frac{1}{2}\mathbb{E}\left[\int_{\mathbb{R}^{n d}}\left|R_{2m-1}(x,A,K)\right|dx1_{E}\right]
+\mathbb{P}\left(\|A-K\|_{\text{op}}>\frac{\lambda(K)}{2}\right)\\
+\sum_{k=1}^{2m-1}\sum_{J\in S_{n d}^{(2k)}}\frac{1}{2^{k+1}k!}\sum_{\alpha\in \mathcal{A}_J}\left|\mathbb{E}\left[(Q^{\oplus n})_{\alpha_1,\alpha_2}\dots (Q^{\oplus n})_{\alpha_{2k-1},\alpha_{2k}}1_{E^C}\right]\right|\cdot\\
\cdot\prod_{i=0}^{n-1}\mathbb{E}\left[|H_{j_{di+1}}(N_{di+1})|\right]\dots \mathbb{E}\left[|H_{j_{di+d}}(N_{di+d})|\right].
\end{multline*}
Recalling the notation $\Gamma_t:=tA+(1-t)K$, $\tilde{Q}_t:=\sqrt{\Gamma_t}^{-1}(A-K)\sqrt{\Gamma_t}^{-1}$, $Q:=\tilde{Q}_0$ and using the formula for the remainder \eqref{resto} together with the fact that $|O_{i,j}|\le\|O\|_{op}$ for every matrix $O\in\mathbb{R}^{n d\times n d}$ and for every $i,j\in \{1,\dots,n d\}$, we obtain that

\begin{multline}\label{uso_card_A_J}
d_{\text{TV}}\left(P_Z,\gamma_{Z,G,m}\right)\\
\le \frac{1}{2}\mathbb{E}\Bigg[1_{E}\int_{\mathbb{R}^{nd}}\Bigg|\frac{1}{2^{2m}(2m-1)!}\int_{0}^1(1-t)^{2m-1}\sum_{J\in S_{nd}^{(4m)}}\sum_{\alpha\in \mathcal{A}_J}(\tilde{Q}_t^{\oplus n})_{\alpha_1,\alpha_2}\dots (\tilde{Q}_t^{\oplus n})_{\alpha_{4m-1},\alpha_{4m}}\cdot\\
\cdot \prod_{i=0}^{n-1}\left(H_{j_{di+1}}(u^{(i)}_1)\dots H_{j_{di+d}}(u^{(i)}_d)\right)_{|_{u^{(i)}=\sqrt{\Gamma_t}^{-1}\tilde{x}_{i+1}}}\phi_{\Gamma_t^{\oplus n}}(x)dt\Bigg|dx\Bigg]\\
+\sum_{k=1}^{2m-1}\sum_{J\in S_{nd}^{(2k)}}\frac{(2k)!}{\sqrt{j_1!\dots j_{n d}!}}\frac{1}{2^{k+1}k!}\mathbb{E}\left[\|Q^{\oplus n}\|_{op}^k1_{E^C}\right]
+\frac{2^{2m}}{\lambda(K)^{2m}}\mathbb{E}\left[\|A-K\|_{\text{op}}^{2m}\right]
\end{multline}
\begin{multline}\label{uso_card_S_2k}
\le \frac{1}{{2^{2m+1}(2m-1)!}}\sum_{J\in S_{nd}^{(4m)}}\sum_{\alpha\in \mathcal{A}_J}\int_{0}^1\mathbb{E}\left[1_{E}|(\tilde{Q}_t^{\oplus n})_{\alpha_1,\alpha_2}|\dots |(\tilde{Q}_t^{\oplus n})_{\alpha_{4m-1},\alpha_{4m}}|\right]\cdot\\
\cdot \prod_{i=0}^{n-1}\mathbb{E}\left[\left| H_{j_{di+1}}(N_{di+1})\dots H_{j_{di+d}}(N_{di+d})\right|\right]dt
+\sum_{k=1}^{2m-1}\binom{2k+nd-1}{nd-1}\cdot\\
\cdot\frac{(2k)!2^{2m-k}}{2^{k+1}k!\lambda(K)^{2m-k}} \mathbb{E}\left[\|A-K\|_{\text{op}}^{4m-2k}\right]^{\frac{1}{2}}\mathbb{E}\left[\|\sqrt{K}^{-1}(A-K)\sqrt{K}^{-1}\|_{op}^{2m}\right]^{\frac{k}{2m}}\\
+\frac{2^{2m}}{\lambda(K)^{2m}}\mathbb{E}\left[\|A-K\|_{\text{op}}^{2m}\right],
\end{multline}
using respectively in inequalities \eqref{uso_card_A_J} and \eqref{uso_card_S_2k} that the cardinality of $\mathcal{A}_J$ is equal to $\frac{(2k)!}{j_1!\dots j_{nd}!}$ and that the cardinality of $S_{nd}^{(2k)}$ is given by $\binom{2k+nd-1}{nd-1}$, see e.g. \cite{Stanley11}.

Hence, using again the explicit expressions for the cardinality of $S_{nd}^{(4m)}$ and of $\mathcal{A}_J$ and bounding any element of a matrix with its norm, 
\begin{multline*}
d_{\text{TV}}\left(P_Z,\gamma_{Z,G,m}\right)\le 
 \frac{(4m)!}{{2^{2m+1}(2m-1)!}}\binom{4m+nd-1}{nd-1}\mathbb{E}\left[1_{E} \int_0^1\|\Gamma_t^{-1}\|_{\text{op}}^{2m}dt\|A-K\|_{\text{op}}^{2m}\right]\\
+\sum_{k=1}^{2m-1}\binom{2k+nd-1}{nd-1}\frac{(2k)!2^{2m}}{2^{2k+1}k!\lambda(K)^{2m}} \mathbb{E}\left[\|A-K\|_{\text{op}}^{4m}\right]^{\frac{1}{2}}
+\frac{2^{2m}}{\lambda(K)^{2m}}\mathbb{E}\left[\|A-K\|_{\text{op}}^{2m}\right]
\end{multline*}
\begin{multline}\label{expl_bound_edg_TV_gen}
\le \frac{(4m)!}{{2(2m-1)!\lambda(K)^{2m}}}\binom{4m+nd-1}{nd-1}\mathbb{E}\left[\|A-K\|_{\text{op}}^{2m}\right]\\
+\sum_{k=1}^{2m-1}\binom{2k+nd-1}{nd-1}\frac{(2k)!2^{2m}}{2^{2k+1}k!\lambda(K)^{2m}}\mathbb{E}\left[\|A-K\|_{\text{op}}^{4m}\right]^{\frac{1}{2}}
+\frac{2^{2m}}{\lambda(K)^{2m}}\mathbb{E}\left[\|A-K\|_{\text{op}}^{2m}\right],
\end{multline}
using in the last inequality that when $\|A-K\|_{op}\le \frac{\lambda(K)}{2}$ then $\|\Gamma_t^{-1}\|_{\text{op}}\le \frac{2}{\lambda(K)}$.
\subsection{Proof of the lower bound \eqref{lower_bound_egd_gen} on the Total Variation distance}
As in the proof of Proposition 5.4 in \cite{FHMNP}, consider the function $h:\mathbb{R}^{n d}\to\mathbb{R}$ defined as
\[
h(x)=\cos(x_{j}).
\]
for a fixed index $j\in\{1,\dots,d\}$.

Since $|h(x)|\le 1$ for every $x\in\mathbb{R}^{nd}$, recalling Definition \ref{dTV_def}, we have that
\begin{equation}\label{lower_d_tv}
d_{\text{TV}}\left(P_Z,\gamma_{Z,G,m}\right)\ge \left|\int_{\mathbb{R}^{n d}}\cos(x_{j})dP_Z(x)-\int_{\mathbb{R}^{nd}}\cos(x_{j})d\gamma_{Z,G,m}(x)\right|.
\end{equation}
Since
\[
\cos(x_{j})=\frac{e^{ix_{j}}+e^{-ix_{j}}}{2},
\]
we have that
\begin{equation}\label{exp_cos_gauss}
\int_{\mathbb{R}^{n d}}\cos(x_{j})dP_Z(x)=\frac{1}{2}\left(\mathbb{E}\left[e^{iZ_{j}}\right]+\mathbb{E}\left[e^{-iZ_{j}}\right]\right)=\mathbb{E}\left[e^{-\frac{A_{j,j}}{2}}\right],
\end{equation}
since $(Z_{1},\dots,Z_{d})$ is conditionally Gaussian with conditional covariance matrix $A$.

Instead, recalling the definition of $\gamma_{Z,G,m}$ in \eqref{edg_measure} 
we have that, denoting 

\noindent
$Q:=\sqrt{K}^{-1}\left(A-K\right)\sqrt{K}^{-1}$, $\tilde{x}_i:=\left(x_{d(i-1)+1},\dots, x_{d(i-1)+d}\right)$ for every $i=1,\dots,n$ and $S_{nd}^{(2k)}$ and $\mathcal{A}_J$ defined as in \eqref{def_S_k} and \eqref{def_A_j} respectively,
\begin{multline*}
\int_{\mathbb{R}^{n d}}\cos(x_{j})d\gamma_{Z,G.m}(x)\\
=\int_{\mathbb{R}^{n d}}\cos(x_{j})\Bigg(
\prod_{i=1}^{n}\phi_K(\tilde{x}_i)
+\sum_{k=2}^{2m-1}\frac{1}{k!\,2^k}
\sum_{J\in S_{nd}^{(2k)}}\sum_{\alpha\in \mathcal{A}_J}
\mathbb{E}\!\left[(Q^{\oplus n})_{\alpha_1,\alpha_2}\cdots
(Q^{\oplus n})_{\alpha_{2k-1},\alpha_{2k}}\right]\cdot\\
\cdot \prod_{i=0}^{n-1}
\left(
H_{j_{di+1}}\!\left(u^{(i)}_1\right)\cdots
H_{j_{di+d}}\!\left(u^{(i)}_d\right)
\right)\Big|_{\,u^{(i)}=\sqrt{K}^{-1}\tilde{x}_{i+1}}\,
\phi_K\!\left(\tilde{x}_{i+1}\right)
\Bigg)dx,
\end{multline*}
since by the assumption that $\mathbb{E}[A]=K$, one has that
\[
\sum_{J\in S_{nd}^{(2)}}\sum_{\alpha\in \mathcal{A}_J}
\mathbb{E}\!\left[(Q^{\oplus n})_{\alpha_1,\alpha_2}\right]=0.
\]

Using Lemma \ref{lemma_der_phi_K_in_0}, we have that
\[ 
\int_{\mathbb{R}^{n d}}\cos(x_{j})d\gamma_{Z,G.m}(x)=\int_{\mathbb{R}^d}\cos(x_{j})\phi_K(\tilde{x}_{1})d\tilde{x}_{1}+\sum_{k=2}^{2m-1}\frac{1}{k!}
\mathbb{E}\left[\int_{\mathbb{R}^{nd}}\cos(x_{j})\frac{\partial^k}{\partial t^k}\phi_{\Gamma_t^{\oplus n}}(x)_{|_{t=0}}dx\right].
\]
 Applying \eqref{der_t_der_x_follow} and arguing as in \eqref{exp_cos_gauss} for the first term, we obtain

\[
\int_{\mathbb{R}^{n d}}\cos(x_{j})d\gamma_{Z,G.m}(x)=e^{-\frac{K_{j,j}}{2}}+\sum_{k=2}^{2m-1}\frac{1}{k!2^k}\mathbb{E}\left[\int_{\mathbb{R}^{nd}}\cos(x_{j})\langle (A^{\oplus n}-K^{\oplus n})^{\otimes k},\nabla^{2k}\phi_{K^{\oplus n}}(x)\rangle dx\right]
\]
\begin{multline*}
=e^{-\frac{K_{j,j}}{2}}+\sum_{k=2}^{2m-1}\frac{1}{k!2^k}\sum_{J\in S_{nd}^{(2k)}}\sum_{\alpha\in\mathcal{A}_J}\int_{\mathbb{R}^{nd}}\cos(x_{j})\mathbb{E}\left[(A^{\oplus n}-K^{\oplus n})_{\alpha_1,\alpha_2}\dots(A^{\oplus n}-K^{\oplus n})_{\alpha_{2k-1},\alpha_{2k}}\right]\cdot\\
\cdot \frac{\partial^{2k}}{\partial x_1^{j_1}\dots\partial x_{nd}^{j_{nd}}}\phi_{K^{\oplus n}}(x)\rangle dx
\end{multline*}
\[
=e^{-\frac{K_{j,j}}{2}}+\sum_{k=2}^{2m-1}\frac{1}{2^kk!}\mathbb{E}\left[(A_{j,j}-K_{j,j})^k\right] \int_{\mathbb{R}^{ d}}\cos(x_{j})\frac{\partial^{2k}}{\partial x_{j}^{2k}}\phi_K(\tilde{x}_{1})d\tilde{x}_{1}
\]
\begin{equation}\label{exp_cos_gamma}
=e^{-\frac{K_{j,j}}{2}}+\sum_{k=2}^{2m-1}\frac{(-1)^k}{2^kk!}\mathbb{E}\left[(A_{j,j}-K_{j,j})^k\right]e^{-\frac{K_{j,j}}{2}},
\end{equation}
doing an integration by parts to obtain the last equality.

Observe now that performing a Taylor expansion of the function $x\in\mathbb{R}\mapsto e^{-\frac{x}{2}}$ around $x_0=K_{j,j}$ and evaluating it in $x=A_{j,j}$, one has that
\begin{equation}\label{taylor_exp_char_gauss_C_1}
e^{-\frac{A_{j,j}}{2}}=e^{-\frac{K_{j,j}}{2}}+\sum_{k=1}^{2m}\frac{(-1)^k\left(A_{j,j}-K_{j,j}\right)^k}{2^kk!}e^{-\frac{K_{j,j}}{2}}+\tilde{R}_{2m}(A_{j,j}),
\end{equation}
where for any $x\in\mathbb{R}$, as in \cite{Apostol67},
\[
\tilde{R}_{2m}(x)=-\frac{1}{(2m)!}\int_0^1(1-t)^{2m}\frac{\left(x-K_{j,j}\right)^{2m+1}}{2^{2m+1}}e^{-\frac{tx+(1-t)K_{j,j}}{2}}dt
\]
is the remainder term.
Substituting \eqref{exp_cos_gauss} and \eqref{exp_cos_gamma} in \eqref{lower_d_tv}, recalling that $\mathbb{E}[A]=K$, we obtain that
\begin{multline*}
d_{\text{TV}}\left(P_Z,\gamma_{Z,G,m}\right)\ge \left|e^{-\frac{A_{j,j}}{2}}-e^{-\frac{K_{j,j}}{2}}-\sum_{k=2}^{2m-1}\frac{(-1)^k}{2^kk!}\mathbb{E}\left[\left(A_{j,j}-K_{j,j}\right)^k\right]e^{-\frac{K_{j,j}}{2}}\right|\\
=\left|\frac{\mathbb{E}\left[\left(A_{j,j}-K_{j,j}\right)^{2m}\right]e^{-\frac{K_{j,j}}{2}}}{2^{2m}(2m)!}+\mathbb{E}\left[\tilde{R}_{2m}(A_{j,j})\right]\right|,
\end{multline*}
using equality \eqref{taylor_exp_char_gauss_C_1} in the last identity.

Hence
\begin{multline}\label{low_bound_d_tv}
d_{\text{TV}}\left(P_Z,\gamma_{Z,G,m}\right)\\
\ge\left|\frac{\mathbb{E}\left[\left(A_{j,j}-K_{j,j}\right)^{2m}\right]e^{-\frac{K_{j,j}}{2}}}{2^{2m}(2m)!}-\mathbb{E}\left[\frac{1}{(2m)!}\int_0^1(1-t)^{2m}\frac{\left(A_{j,j}-K_{j,j}\right)^{2m+1}}{2^{2m+1}}e^{-\frac{tA_{j,j}+(1-t)K_{j,j}}{2}}dt\right]\right|\\
\ge \frac{\mathbb{E}\left[\left(A_{j,j}-K_{j,j}\right)^{2m}\right]e^{-\frac{K_{j,j}}{2}}}{2^{2m}(2m)!}-\frac{1}{(2m)!}\left|\mathbb{E}\left[\int_0^1(1-t)^{2m}\frac{\left(A_{j,j}-K_{j,j}\right)^{2m+1}}{2^{2m+1}}e^{-\frac{tA_{j,j}+(1-t)K_{j,j}}{2}}dt\right]\right|\\
\ge \frac{\mathbb{E}\left[\left(A_{j,j}-K_{j,j}\right)^{2m}\right]e^{-\frac{K_{j,j}}{2}}}{2^{2m}(2m)!}-\frac{\mathbb{E}\left[\left|A_{j,j}-K_{j,j}\right|^{2m+1}\right]}{2^{2m+1}(2m)!}.
\end{multline}

\section{Proof of Theorem \ref{edg_nn}}\label{proof_th_edg_TV}
\subsection{Upper bound on the Total Variation distance}

The proof of the upper bound \eqref{bound_dTV_nn_no_exp} is a direct application of the general bounds \eqref{bound_dTV_gen_no_exp} together with the following Remark.

\begin{remark}\label{upp_bound_var_nn}
    Thanks to Theorem 10 in \cite{CP25}, which is a direct consequence of Theorem 7.3, Proposition 7.4 and Lemma 7.5 in \cite{Han_Gas}, one has that for every $i,j\in\{1,\dots, d\}$ and $p\in\mathbb{N}$,
    \[
    \mathbb{E}\left[\left(A^{(L+1)}_{i,j}-\mathbb{E}\left[A^{(L+1)}_{i,j}\right]\right)^{2p}\right]\le \frac{D_2(p)}{n^p},
    \]
    where $D_2(p)>0$ is a constant independent of $n,n_1,\dots, n_L$.

    Moreover, as in Proposition 2 from \cite{CP25} (an application of Proposition 10.3 from  \cite{Han_Gas}), one has that  for every $i,j\in\{1,\dots, d\}$
    \begin{equation}\label{upp_bound_diff_exp}
\Big|\mathbb{E}[A_{i,j}^{(L+1)}]-K_{i,j}^{(L+1)}\Big|\le \frac{D_3}{n},
\end{equation}
 where $D_3>0$ is a constant independent of $n$.
\end{remark}
\subsection{Lower bound on the Total Variation distance}
The proof of the lower bound \eqref{lower_bound_egd_nn} uses the following Lemma, which is proved in Subsection \ref{sec_proof_lemma_low_bound_mom_nn}.
\begin{lemma}\label{low_bound_var_nn}
    In the case of the neural network evaluated in one input $x^{(j)}\in\mathbb{R}^{n_0}\setminus \{0\}$ with $n_1\asymp n_2\asymp\dots\asymp n_L\asymp n$, one has that for every $p\in\mathbb{N}$
    \[
    \mathbb{E}\left[\left(A^{(L+1)}_{j,j}-\mathbb{E}\left[A^{(L+1)}_{j,j}\right]\right)^{2p}\right]\ge\frac{D_4(p)}{n^{{p}}},
    \]
    where $D_4(p)>0$ is a constant independent of $n,n_1,\dots,n_L$.
    \end{lemma}

    Applying Lemma \ref{low_bound_var_nn} and Remark \ref{upp_bound_var_nn} to the lower bound \eqref{lower_bound_egd_gen}, one has that for every $i=1,\dots,d$
    \begin{multline*}
d_{\text{TV}}\left(P_Z,\gamma_{Z,G,m}\right)
\ge\frac{1}{2^{2m}(2m)!}\Bigg({\mathbb{E}\left[\left({A^{(L+1)}_{j,j}}-{\mathbb{E}\left[A^{(L+1)}_{j,j}\right]}\right)^{2m}\right]e^{-\frac{{K^{(L+1)}_{j,j}}}{2}}}\\
-\frac{1}{2}\mathbb{E}\left[\left|{A^{(L+1)}_{j,j}}-{\mathbb{E}\left[A^{(L+1)}_{j,j}\right]}\right|^{2m+1}\right]\Bigg)\\
\ge \frac{1}{2^{2m}(2m)!}\left(\frac{D_4(m)e^{-\frac{{K^{(L+1)}_{j,j}}}{2}}}{n^m}-\frac{1}{2}\frac{\left(D_2\left(m+1\right)\right)^{\frac{2m+1}{2m+2}}}{n^{m+\frac{1}{2}}}\right).
    \end{multline*}
   
    Hence
    \begin{equation}\label{d_tv_low_n}
    d_{\text{TV}}\left(P_Z,\gamma_{Z,G,m}\right)\ge \frac{1}{n^m2^{2m}(2m)!}\left({D_4(m)e^{-\frac{K_{j,j}}{2}}}-\frac{1}{2}\frac{\left(D_2\left(m+1\right)\right)^{\frac{2m+1}{2m+2}}}{n^{\frac{1}{2}}}\right)
    \ge \frac{C}{n^m},
    \end{equation}
    where $C>0$ is a constant independent of $n$ and the last inequality holds for $n>N\in\mathbb{N}$ such that
    \[
    {D_4(m)e^{-\frac{K_{j,j}}{2}}}-\frac{1}{2}\frac{\left(D_2\left(m+1\right)\right)^{\frac{2m+1}{2m+2}}}{N^{\frac{1}{2}}}>0.
    \]
{
If instead $n<N$, one has that
\[
  d_{\text{TV}}\left(P_{Z_n},\gamma_{Z_n,G,m}\right)\ge \min_{k<N} d_{\text{TV}}\left(P_{Z_k},\gamma_{Z_k,G,m}\right)>\frac{\min_{k<N} d_{\text{TV}}\left(P_{Z_k},\gamma_{Z_k,G,m}\right)}{n},
\]
where we denoted the neural network with $Z_n$ to stress the order of inner width $n\in\mathbb{N}$. Moreover, for $L\ge 1$ and $\sigma$ not constant, 
\[
\min_{k<N} d_{\text{TV}}\left(P_{Z_k},\gamma_{Z_k,G,m}\right)>0,
\]
thus concluding the proof.
}

\section{Proof of technical Lemmas}\label{add_lemma}

\subsection{Proof of Lemma \ref{prop_herm_with_expectation}}\label{proof_norm_const}
\subsubsection{Proof of identity \eqref{eq_1_lemma_int_her}}
We are going to use the following explicit expression (proved in \cite{AS65}) for the Hermite polynomials defined in \eqref{herm_def}: for every $n\in\mathbb{N}$ and $x\in\mathbb{R}$
\[
H_n(x)=n!\sum_{m=0}^{\lfloor \frac{n}{2}\rfloor}\frac{(-1)^mx^{n-2m}}{m!(n-2m)!2^m}.
\]
Hence
\begin{multline*}
\mathbb{E}\left[H_{j_1}\left(\left(\sqrt{K}^{-1}\tilde{G}\right)_1\right)\dots H_{j_d}\left(\left(\sqrt{K}^{-1}\tilde{G}\right)_d\right)\right]\\
=j_1!\dots j_d!\sum_{m_1=0}^{\lfloor \frac{j_1}{2}\rfloor}\dots \sum_{m_d=0}^{\lfloor \frac{j_d}{2}\rfloor}\frac{(-1)^{m_1+\dots+m_d}}{m_1!\dots m_d!(j_1-2m_1)!\dots (j_d-2m_d)!2^{m_1+\dots+m_d}}\cdot\\
\cdot \mathbb{E}\left[\left((\sqrt{K}^{-1}\tilde{G})_1\right)^{j_1-2m_1}\dots \left((\sqrt{K}^{-1}\tilde{G})_d\right)^{j_d-2m_d}\right].
\end{multline*}
Writing $\sqrt{K}^{-1}\tilde{G}\stackrel{\text{law}}{=}\tilde{Z}+\sqrt{K}^{-1}\mu$, with $\tilde{Z}\sim\mathcal{N}_d\left(0,\sqrt{K}^{-1}\Sigma\sqrt{K}^{-1}\right)$, and using the binomial formula,
\begin{multline*}
    \mathbb{E}\left[H_{j_1}\left(\left(\sqrt{K}^{-1}\tilde{G}\right)_1\right)\dots H_{j_d}\left(\left(\sqrt{K}^{-1}\tilde{G}\right)_d\right)\right]\\
=j_1!\dots j_d!\sum_{m_1=0}^{\lfloor \frac{j_1}{2}\rfloor}\dots \sum_{m_d=0}^{\lfloor \frac{j_d}{2}\rfloor}\frac{(-1)^{m_1+\dots+m_d}}{m_1!\dots m_d!(j_1-2m_1)!\dots (j_d-2m_d)!2^{m_1+\dots+m_d}}\cdot\\
\cdot\sum_{r_1=0}^{j_1-2m_1}\dots\sum_{r_d=0}^{j_d-2m_d}\binom{j_1-2m_1}{r_1} \dots\binom{j_d-2m_d}{r_d}\left((\sqrt{K}^{-1}\mu)_1\right)^{r_1}\dots \left((\sqrt{K}^{-1}\mu)_d\right)^{r_d}\cdot\\
\cdot\mathbb{E}\left[(\tilde{Z}_1)^{j_1-2m_1-r_1}\dots (\tilde{Z}_d)^{j_d-2m_d-r_d}\right]
\end{multline*}
\begin{multline*}
    =\sum_{r_1=j_1-2\lfloor\frac{j_1}{2}\rfloor}^{j_1}\dots \sum_{r_d=j_d-2\lfloor\frac{j_d}{2}\rfloor}^{j_d}\binom{j_1}{r_1}\dots\binom{j_d}{r_d}\left((\sqrt{K}^{-1}\mu)_1\right)^{r_1}\dots \left((\sqrt{K}^{-1}\mu)_d\right)^{r_d} \cdot\\
\cdot\mathbb{E}\left[H_{j_1-r_1}(\tilde{Z}_1)\dots H_{j_d-r_d}(\tilde{Z}_d)\right],
\end{multline*}
which prove the first identity \eqref{eq_1_lemma_int_her}.

\subsubsection{Proof of identity \eqref{eq_2_lemma_int_her}}

We prove identity \eqref{eq_2_lemma_int_her} assuming $v=2k$ for $k\in\mathbb{N}$ (the proof of the case $v=2k+1$ follows analogously). The strategy is to do an induction over $k\in\mathbb{N}$ and to use that for every $n\in\mathbb{N}$ and $x\in\mathbb{R}$ it holds that
\begin{equation}\label{der_her}
H_{n+1}^{'}(x)=(n+1)H_n(x)
\end{equation}
and that
\begin{equation}\label{her_rico}
H_{n+1}(x)=xH_n(x)-H^{'}_n(x).
\end{equation}
When $k=1$, $s=(s_1,s_2,\dots,s_d)$ with $s_1+s_2+\dots+s_d=2$, which means that the $(s_i)_i$ are all zero apart for at most two elements. We assume without loss of generality that $s_3=s_4=\dots=s_d=0$. Recalling that $H_1(x)=x$ and $H_0(x)=1$, when $s_1=2$ (analogously when $s_2=2$) one has that
\[ 
\frac{1}{2}\mathbb{E}\left[H_2\left(\tilde{Z}_1\right)\right]=\frac{1}{2}\mathbb{E}\left[\tilde{Z}_1 H_1\left(\tilde{Z}_1\right)-1\right]
=\frac{1}{2}\left(\left(\sqrt{K}^{-1}\Sigma\sqrt{K}^{-1}\right)_{1,1}-1\right)=\frac{1}{2}\sum_{\alpha\in \mathcal{A}_{(2,0,\dots,0)}}R_{\alpha_1,\alpha_2}.
\]
For $s_1=s_2=1$ instead
\[
\mathbb{E}\left[H_1\left(\tilde{Z}_1\right)H_1\left(\tilde{Z}_2\right)\right]=\left(\sqrt{K}^{-1}\Sigma\sqrt{K}^{-1}\right)_{1,2}=\frac{1}{2}\sum_{\alpha\in \mathcal{A}_{(1,1)}}R_{\alpha_1,\alpha_2},
\]
concluding the proof of \eqref{eq_2_lemma_int_her} for $k=1$ (equivalently for $v=2)$.

We assume now that equation  \eqref{eq_2_lemma_int_her} holds for $v=2(k-1)$ and we prove it for $v=2k$ with $k>1$. Assuming that $s_1\ge 1$ (otherwise the proof follow analogously since at least one element in $s$ is not zero) and using again identities \eqref{der_her} and \eqref{her_rico} we obtain that
\begin{multline}\label{aft_prop_herm_der_iter}
\mathbb{E}\left[H_{s_1}\left(\tilde{Z}_1\right)\dots H_{s_d}\left(\tilde{Z}_d\right)\right]
=\mathbb{E}\left[\tilde{Z}_1H_{s_1-1}\left(\tilde{Z}_1\right)H_{s_2}\left(\tilde{Z}_2\right)\dots H_{s_d}\left(\tilde{Z}_d\right)\right]\\
-(s_1-1)\mathbb{E}\left[H_{s_1-2}\left(\tilde{Z}_1\right)H_{s_2}\left(\tilde{Z}_2\right)\dots H_{s_d}\left(\tilde{Z}_d\right)\right].
\end{multline}
Writing $\tilde{Z}\stackrel{\text{law}}{=}\sqrt{K^{-1}\Sigma}N$, with $N\sim\mathcal{N}_d(0,I_d)$, recalling that $\Sigma$ and $K$ are symmetric matrices, and using the Gaussian integration by parts (Lemma 3.1.2 in \cite{NP12}) it follows that
\begin{multline*}
\mathbb{E}\left[\tilde{Z}_1H_{s_1-1}\left(\tilde{Z}_1\right)H_{s_2}\left(\tilde{Z}_2\right)\dots H_{s_d}\left(\tilde{Z}_d\right)\right]\\
=\sum_{r=1}^d\left(\sqrt{K^{-1}\Sigma}\right)_{1,r}\mathbb{E}\left[N_rH_{s_1-1}\left((\sqrt{K^{-1}\Sigma}N)_1\right)H_{s_2}\left((\sqrt{K^{-1}\Sigma}N)_2\right)\dots H_{s_d}\left((\sqrt{K^{-1}\Sigma}N)_d\right)\right]\\
=\sum_{r=1}^d\left(\sqrt{K^{-1}\Sigma}\right)_{1,r}\left(\sqrt{K^{-1}\Sigma}\right)_{1,r}\cdot\\
\cdot\mathbb{E}\left[H^{'}_{s_1-1}\left((\sqrt{K^{-1}\Sigma}N)_1\right)H_{s_2}\left((\sqrt{K^{-1}\Sigma}N)_2\right)\dots H_{s_d}\left((\sqrt{K^{-1}\Sigma}N)_d\right)\right]\\
+\sum_{i=2}^d\sum_{r=1}^d\left(\sqrt{K^{-1}\Sigma}\right)_{1,r}\left(\sqrt{K^{-1}\Sigma}\right)_{i,r}\cdot\\
\cdot \mathbb{E}\left[H_{s_1-1}\left((\sqrt{K^{-1}\Sigma}N)_1\right)H^{'}_{s_i}\left((\sqrt{K^{-1}\Sigma}N)_i\right)\prod_{u=2,u\ne i}^dH_{s_u}\left((\sqrt{K^{-1}\Sigma}N)_u\right)\right]
\end{multline*}
\begin{multline*}
=(s_1-1)\left({\sqrt{K^{-1}}\Sigma\sqrt{K^{-1}}}\right)_{1,1}\cdot\\
\cdot\mathbb{E}\left[H_{s_1-2}\left((\sqrt{K^{-1}\Sigma}N)_1\right)H_{s_2}\left((\sqrt{K^{-1}\Sigma}N)_2\right)\dots H_{s_d}\left((\sqrt{K^{-1}\Sigma}N)_d\right)\right]\\
+\sum_{i=2}^ds_i\left(\sqrt{K^{-1}}{\Sigma}\sqrt{K^{-1}}\right)_{1,i}\cdot\\
\cdot \mathbb{E}\left[H_{s_1-1}\left((\sqrt{K^{-1}\Sigma}N)_1\right)H_{s_i-1}\left((\sqrt{K^{-1}\Sigma}N)_i\right)\prod_{u=2,u\ne i}^dH_{s_u}\left((\sqrt{K^{-1}\Sigma}N)_u\right)\right].
\end{multline*}
Using now the inductive assumption for $v=2(k-1)$, we obtain that
\begin{multline*}
    \mathbb{E}\left[\tilde{Z}_1H_{s_1-1}\left(\tilde{Z}_1\right)H_{s_2}\left(\tilde{Z}_2\right)\dots H_{s_d}\left(\tilde{Z}_d\right)\right]\\
    =\left(\sqrt{K}^{-1}{\Sigma}\sqrt{K}^{-1}\right)_{1,1}\frac{(s_1-1)!\prod_{u=2}^ds_u!}{2^{k-1}(k-1)!}\sum_{\alpha\in \mathcal{A}_{(s_1-2,s_2,\dots,s_d)}}R_{\alpha_1,\alpha_2}\dots R_{\alpha_{2k-3},\alpha_{2k-2}}\\
+\sum_{i=2}^d\left(\sqrt{K}^{-1}{\Sigma}\sqrt{K}^{-1}\right)_{1,i}\frac{(s_1-1)!\prod_{u=2}^ds_u!}{2^{k-1}(k-1)!}\sum_{\beta\in \mathcal{A}_{(s_1-1,s_2\dots s_i-1,\dots,s_d)}}R_{\beta_1,\beta_2}\dots R_{\beta_{2k-3},\beta_{2k-2}}.
\end{multline*}
Recalling that $R:=\sqrt{K}^{-1}\left(\Sigma-K\right)\sqrt{K}^{-1}$ and using identity \eqref{aft_prop_herm_der_iter},
\begin{multline}\label{after_induct_hyp}
\mathbb{E}\left[H_{s_1}\left(\tilde{Z}_1\right)\dots H_{s_d}\left(\tilde{Z}_d\right)\right]
=\frac{(s_1-1)!\prod_{u=2}^ds_u!}{2^{k-1}(k-1)!}\sum_{\alpha\in \mathcal{A}_{(s_1-2,s_2,\dots,s_d)}}R_{1,1}R_{\alpha_1,\alpha_2}\dots R_{\alpha_{2k-3},\alpha_{2k-2}}\\
+\frac{(s_1-1)!\prod_{u=2}^ds_u!}{2^{k-1}(k-1)!}\sum_{i=2}^d\sum_{\beta\in \mathcal{A}_{(s_1-1,s_2\dots s_i-1,\dots,s_d)}}{R}_{1,i}R_{\beta_1,\beta_2}\dots R_{\beta_{2k-3},\beta_{2k-2}}.
\end{multline}
Observe now that
\[
R_{1,1}R_{\alpha_1,\alpha_2}\dots R_{\alpha_{2k-3},\alpha_{2k-2}}\in \mathcal{A}_{(s_1,\dots,s_d)}\quad\text{for $\alpha\in \mathcal{A}_{(s_1-2,s_2,\dots,s_d)}$}
\]
and
\[
{R}_{1,i}R_{\beta_1,\beta_2}\dots R_{\beta_{2k-3},\beta_{2k-2}}\in \mathcal{A}_{(s_1,\dots,s_d)}\quad\text{for $\beta\in \mathcal{A}_{(s_1-1,s_2\dots s_i-1,\dots,s_d)}$}.
\]
Fix now the sequence
\[
\tilde{\alpha}:=\bigl(\underbrace{1,\dots,1}_{s_1},\dots,\underbrace{d,\dots,d}_{s_d}\bigr),
\]
then every element in $\mathcal{A}_{(s_1,\dots,s_d)}$ can be seen as a permutation of $\tilde{\alpha}$.

Denoting with ${\Pi}_{2k}$ the set of permutations of $2k$ elements, from \eqref{after_induct_hyp} it follows that
\begin{multline*}
\mathbb{E}\left[H_{s_1}\left(\tilde{Z}_1\right)\dots H_{s_d}\left(\tilde{Z}_d\right)\right]\\
=\frac{1}{2^{k-1}(k-1)!}\sum_{\sigma\in \Pi_{2k}:\sigma(1)=1}R_{\tilde{\alpha}_{\sigma(1)},\tilde{\alpha}_{\sigma(2)}}\dots R_{\tilde{\alpha}_{\sigma(2k-1)},\tilde{\alpha}_{\sigma(2k)}}\\
=\frac{1}{2^{k}k!}\sum_{\sigma\in \Pi_{2k}}R_{\tilde{\alpha}_{\sigma(1)},\tilde{\alpha}_{\sigma(2)}}\dots R_{\tilde{\alpha}_{\sigma(2k-1)},\tilde{\alpha}_{\sigma(2k)}}=\frac{\prod_{u=1}s_u!}{2^{k}k!}\sum_{\alpha\in \mathcal{A}_s}R_{{\alpha}_{1},{\alpha}_{2}}\dots R_{{\alpha}_{2k-1},{\alpha}_{2k}}
\end{multline*}
where the second equality follows from the invariance of the summand under permutations of the $2k$ indices.

\subsection{Proof of Lemma \ref{lemma_der_phi_K_in_0}}\label{sec_proof_lemma_der_phi_K_in_0}
Observe first that 
\[
(\Gamma_t)^{\oplus n}=tA^{\oplus n}+(1-t)K^{\oplus n}.
\]
Then, analogously as shown in Remark 28 from \cite{CP25}, one has that 
\begin{equation}\label{der_t_der_x}
\frac{\partial}{\partial t}\phi_{tA^{\oplus n}+(1-t)K^{\oplus n}}(x)=\frac{1}{2}\operatorname{tr}\left((A^{\oplus n}-K^{\oplus n})\nabla^2\phi_{tA^{\oplus n}+(1-t)K^{\oplus n}}(x)\right),
\end{equation}
and hence, as in the proof of Proposition 4 in \cite{CP25},
\begin{equation}\label{der_t_der_x_follow}
\frac{\partial^k}{\partial t^k}\phi_{tA^{\oplus n}+(1-t)K^{\oplus n}}(x)=\frac{1}{2^k}\langle (A^{\oplus n}-K^{\oplus n})^{\otimes k},\nabla^{2k}\phi_{(\Gamma_t)^{\oplus n}}(x)\rangle\\
\end{equation}
where 
for any matrix $M\in\mathbb{R}^{(nd)\times (n d)}$, $M^{\otimes k}\in \mathbb{R}^{((n d)\times (n d))^k}$ and it is defined as
\[
\left(M^{\otimes k}\right)_{(i_1,\dots,i_{k}),(j_1,\dots,j_{k})}:=M_{i_1,j_1}\dots M_{i_{k},j_{k}}
\]
for indexes $i_1,\dots,i_{k},j_1,\dots,j_k\in \{1,\dots,nd\}$.

Therefore, using the chain rule and the fact that 
\[
\sqrt{\operatorname{det}((\Gamma_t)^{\oplus n})}\phi_{(\Gamma_t)^{\oplus n}}(x)=\phi_{I_{nd}}(\sqrt{(\Gamma_t)^{\oplus n}}^{-1}x),
\]
observing that $\sqrt{(\Gamma_t)^{\oplus n}}^{-1}=(\sqrt{\Gamma_t}^{\oplus n})^{-1}=(\sqrt{\Gamma_t}^{-1})^{\oplus n}$,
one obtains that
\[ 
\frac{\partial}{\partial t}\phi_{(\Gamma_t)^{\oplus n}}(x)=\frac{1}{2^k}\left\langle \left(\left(\sqrt{\Gamma_t}^{-1}\left(A-K\right)\sqrt{\Gamma_t}^{-1}\right)^{\oplus n}\right)^{\otimes k},\frac{\nabla^{2k}\phi_{I_{nd}}}{\phi_{I_{nd}}}\left(\left(\sqrt{\Gamma_t}^{-1}\right)^{\oplus n}x\right)\right\rangle\phi_{(\Gamma_t)^{\oplus n}}(x)
\]
Again as in proof of Proposition 4 in \cite{CP25}, introduce the notations of $S_{nd}^{(2k)}$ (see \eqref{def_S_k}) and $\mathcal{A}_J$ (see \eqref{def_A_j}) for any $J\in S_{nd}^{(2k)}$ and denote
\[
\tilde{Q}_t:=\sqrt{\Gamma_t}^{-1}\left(A-K\right)\sqrt{\Gamma_t}^{-1}.
\]

Then
    \begin{multline*}
  \frac{\partial}{\partial t}\phi_{(\Gamma_t)^{\oplus n}}(x)\\
  =\frac{1}{2^k}\sum_{i_1=1}^d\dots\sum_{i_{2k}=1}^d(\tilde{Q}_t^{\oplus n})_{i_1,i_2}\dots (\tilde{Q}_t^{\oplus n})_{i_{2k-1},i_{2k}}\left(\frac{1}{\phi_{I_{nd}}}\frac{\partial^{2k}\phi_{I_{nd}}}{\partial x_{i_1}\dots,\partial x_{i_{2k}}}\right)\left(\left(\sqrt{\Gamma_t}^{-1}\right)^{\oplus n}x\right)\phi_{(\Gamma_t)^{\oplus n}}(x)
\end{multline*}
\begin{multline*}
    =\frac{1}{2^k}\sum_{J\in S_{nd}^{(2k)}}\sum_{\alpha\in \mathcal{A}_J}(\tilde{Q}_t^{\oplus n})_{\alpha_1,\alpha_2}\dots (\tilde{Q}_t^{\oplus n})_{\alpha_{2k-1},\alpha_{2k}}\left(\frac{1}{\phi_{I_{nd}}}\frac{\partial^{2k}\phi_{I_{nd}}}{\prod_{i=0}^{n-1}\partial x_{i+1,1}^{j_{di+1}}\dots,\partial x_{i+1,d}^{j_{id+d}}}\right)\cdot\\
    \cdot\left(\left(\sqrt{\Gamma_t}^{-1}\right)^{\oplus n}x\right)\phi_{(\Gamma_t)^{\oplus n}}(x)
\end{multline*}

\begin{multline*}
    =\frac{1}{2^k}\sum_{J\in S_{nd}^{(2k)}}\sum_{\alpha\in \mathcal{A}_J}(\tilde{Q}_t^{\oplus n})_{\alpha_1,\alpha_2}\dots (\tilde{Q}_t^{\oplus n})_{\alpha_{2k-1},\alpha_{2k}}\left(\prod_{i=0}^{n-1}\left(H_{j_{di+1}}(u^{(i)}_1)\dots H_{j_{di+d}}(u^{(i)}_d)\right)_{|_{u^{(i)}=\sqrt{\Gamma_t}^{-1}\tilde{x}_{i+1}}}\right)\cdot\\
    \cdot \phi_{(\Gamma_t)^{\oplus n}}(x)
\end{multline*}
using the definition of the Hermite polynomials \eqref{herm_def}.
\subsection{Proof of Lemma \ref{low_bound_var_nn}}\label{sec_proof_lemma_low_bound_mom_nn}
Observe that, according to the formula \eqref{cov_z}, one has that 
\[
    A^{(L+1)}_{j,j}=C_b+\frac{C_W}{n_L}\sum_{i=1}^{n_L}\sigma^2\left(z_i^{(L)}(x^{(j)})\right).
    \]
    Hence
    \[
    \mathbb{E}\left[\left(A^{(L+1)}_{j,j}-\mathbb{E}\left[A^{(L+1)}_{j,j}\right]\right)^{2p}\right]= \mathbb{E}\left[\left(\frac{C_W}{n_L}\sum_{i=1}^{n_L}\left(\sigma^2\left(z_i^{(L)}(x^{(j)})\right)-\mathbb{E}\left[\sigma^2\left(z_i^{(L)}(x^{(j)})\right)\right]\right)\right)^{2p}\right]
    \]
    \[
    \ge  \mathbb{E}\left[\left(\frac{C_W}{n_L}\sum_{i=1}^{n_L}\left(\sigma^2\left(z_i^{(L)}(x^{(j)})\right)-\mathbb{E}\left[\sigma^2\left(z_i^{(L)}(x^{(j)})\right)\right]\right)\right)^{2}\right]^p
    \]
    \[
    =\left(\frac{C_W}{n_L}\mathbb{E}\left[\left(\sigma^2\left(z_1^{(L)}(x^{(j)})\right)-\mathbb{E}\left[\sigma^2\left(z_1^{(L)}(x^{(j)})\right)\right]\right)^{2}\right]\right)^p
    \]
    using that $\{z_i^{(L)}(x^{(j)})\}_{i=1,\dots,n_L}$ are independent and identically distributed after conditioning with respect to the $\sigma$-field generated by all the weights and biases up to the layer $L-1$ (see Lemma \ref{gaus_stut_NN}). 

    As a consequence of Proposition 10.3 in \cite{Han_Gas}, one has that as $n\to\infty$
    \[
    \operatorname{Var}\left(\sigma^2\left(z_1^{(L)}(x^{(j)})\right)\right)\to  \operatorname{Var}\left(\sigma^2\left(G_1^{(L)}(x^{(j)})\right)\right).
    \]
    Assuming $\sigma$ not constant and $x\ne (0,\dots,0)\in\mathbb{R}^{n_0}$ then 
    \[
    \operatorname{Var}\left(\sigma^2\left(G_1^{(L)}(x^{(j)})\right)\right)>0
    \]
    and it is independent of $n$, therefore  we have that there exist $N\in\mathcal{N}$ and a constant $D_1>0$ such that for every $n\ge N$ 
    \[
    \operatorname{Var}\left(\sigma^2\left(z_1^{(L)}(x^{(j)})\right)\right)>D_1
    \]
    and hence 
    \[
     \mathbb{E}\left[\left(A^{(L+1)}_{j,j}-\mathbb{E}\left[A^{(L+1)}_{j,j}\right]\right)^{2p}\right]>\frac{C_W^pD_1^p}{n_L^p}.
    \]
    For every $n<N$, instead, we have that
    \[
    \mathbb{E}\left[\left(A^{(L+1)}_{j,j}-\mathbb{E}\left[A^{(L+1)}_{j,j}\right]\right)^{2p}\right]>\frac{C_W^pD_1^p}{n_L^p}
    \]

\section{Edgeworth expansion}\label{app:edg_exp}
Following the approach developed in Chapter 5 in \cite{Mc87}, the Edgeworth expansion consists of an approximation of a joint distribution, which is given by a Gaussian density, multiplied by a sum of correction terms that depend on the cumulants and on the Hermite polynomials defined in \eqref{herm_def}.

More in details, call $\psi_X$ and $\psi_Y$ respectively the densities (assuming that they exist) of a random vector  $X=(X_1,\dots,X_N)$  and of a Gaussian vector $Y=(Y_1,\dots,Y_N)$ with values in $\mathbb{R}^{N}$ with $N\in\mathbb{N}$. For every $m:=(m_1,\dots,m_N)\in\mathbb{N}^N$, denote with $K_m(X)$ and $K_m(Y)$ respectively the $m$-cumulants of $X$ and $Y$, which are defined as
\begin{equation}\label{def_cum}
K_m(X)=(-1)^{m_1+\dots+m_N}\frac{\partial^{m_1} }{\partial {t_1}^{m_1}}\dots\frac{\partial^{m_N} }{\partial {t_N}^{m_N}}\log\mathbb{E}\left[e^{i\sum_{j=1}^Nt_jX_j}\right],
\end{equation}
and analogously is defined $K_m(Y)$.

From the definition of cumulants \eqref{def_cum}, it follows that the cumulant generating functions can be expressed as
\begin{equation}\label{cum_gen_func_X}
\log\mathbb{E}\left[e^{\sum_{j=1}^N t_jX_j}\right]=\sum_{|m|=1}^{\infty}\frac{K_m(X)}{k!}{t_1}^{m_1}\dots{t_N}^{m_N}
\end{equation}
and 
\begin{equation}\label{cum_gen_func_Y}
\log\mathbb{E}\left[e^{\sum_{j=1}^N t_jY_j}\right]=\sum_{|m|=1}^{\infty}\frac{K_m(Y)}{k!}{t_1}^{m_1}\dots{t_N}^{m_N}
\end{equation}
where the sums are taken over all the vectors $m=(m_1,\dots,m_N)\in\mathbb{R}^N$ with $|m|:=m_1+\dots+m_N$.

Subtracting \eqref{cum_gen_func_X} and \eqref{cum_gen_func_Y}, one obtains that 
\[
\log\mathbb{E}\left[e^{\sum_{j=1}^N t_jX_j}\right]=\log\mathbb{E}\left[e^{\sum_{j=1}^N t_jY_j}\right]+\sum_{|m|=1}^{\infty}\frac{(K_m(X)-K_m(Y))}{k!}{t_1}^{m_1}\dots{t_N}^{m_N}.
\]
Then, taking the exponential, 
\begin{multline}\label{form_mom_in_function}
\mathbb{E}\left[e^{\sum_{j=1}^N t_jX_j}\right]=\mathbb{E}\left[e^{\sum_{j=1}^N t_jY_j}\right]\exp\left(\sum_{|m|=1}^{\infty}\frac{(K_m(X)-K_m(Y))}{k!}{t_1}^{m_1}\dots{t_N}^{m_N}\right)\\
=\mathbb{E}\left[e^{\sum_{j=1}^N t_jY_j}\right]\left(\sum_{|m|=0}^{M}\frac{\eta_m}{m!}\,{t_1}^{m_1}\dots{t_N}^{m_N}+R_M\right),
\end{multline}
obtaining the last identity expanding the exponential term and calling with $\{\eta_m\}_m$ its coefficients, which are considered as formal moments, and $R_M$ the remainder.

The Edgeworth expansion of $\psi_X$ is finally given by the inversion of the approximate integral transform in \eqref{form_mom_in_function} term by term:
\begin{equation}\label{edg_standard}
\sum_{|m|=0}^{M}\frac{(-1)^{|m|}}{|m|!}\eta_m\frac{\partial^{|m|}}{\partial {x_1}^{m_1}\dots\partial {x_N}^{m_N}}\psi_Y(x).
\end{equation}

\begin{example}
    If, like in our case, $X$ is a conditionally Gaussian vector with conditional covariance matrix $A^{\oplus n}$  and $Y$ is a Gaussian vector with invertible covariance matrix $K^{\oplus n}$, where $A$ and $K$ are matrices in $\mathbb{R}^{d\times d}$, and if both $X$ and $Y$ have values in $\mathbb{R}^{nd}$, we can explicitly write the formal moments $\{\eta_m\}_m$. In fact, for every $t\in\mathbb{R}^{nd}$,
    \[
    \mathbb{E}\left[e^{\langle t,X\rangle}\right]=\mathbb{E}\left[e^{\frac{1}{2}\langle t,A^{\oplus n}t\rangle}\right]
    \]
    and
    \[
    \mathbb{E}\left[e^{\frac{1}{2}\langle t,Y\rangle}\right]=e^{\langle t,K^{\oplus n}t\rangle},
    \]
    therefore
    \[
    \mathbb{E}\left[e^{\langle t,X\rangle}\right]=\mathbb{E}\left[e^{\langle t,Y\rangle}\right]\mathbb{E}\left[e^{\frac{1}{2}\langle t,(A-K)^{\oplus n}t\rangle}\right],
    \]
    observing that $A^{\oplus n}-K^{\oplus n}=(A-K)^{\oplus n}$. We can now easily expand $e^{\frac{1}{2}\langle t,(A-K)^{\oplus n}t\rangle}$ obtaining 
    \[
     \mathbb{E}\left[e^{\langle t,X\rangle}\right]=\mathbb{E}\left[e^{\langle t,Y\rangle}\right]\left(\sum_{k=0}^M\frac{1}{2^k k!}\mathbb{E}\left[\left({\langle t,(A-K)^{\oplus n}t\rangle}\right)^k\right]+R_M\right),
     \]
     where $R_M$ is the remaining term. Hence
     \begin{multline*}
     \mathbb{E}\left[e^{\langle t,X\rangle}\right]=\mathbb{E}\left[e^{\langle t,Y\rangle}\right]\left(\sum_{k=0}^M\frac{1}{2^k k!}\mathbb{E}\left[\left(\sum_{i=1}^{n d}\sum_{j=1}^{n d}t_{i}t_{j}((A-K)^{\oplus n})_{i,j}\right)^k\right]+R_M\right)
     =\mathbb{E}\left[e^{\langle t,Y\rangle}\right]\cdot\\
     \cdot \left(\sum_{k=0}^M\frac{1}{2^k k!}\sum_{i_1=1}^{nd}\sum_{j_1=1}^{nd}\dots\sum_{i_k=1}^{nd}\sum_{j_k=1}^{nd}t_{i_1}t_{j_1}\dots t_{i_k}t_{j_k}\mathbb{E}\left[((A-K)^{\oplus n})_{i_1,j_1}\dots((A-K)^{\oplus n})_{i_k,j_k}\right]+R_M\right)\\
     =\mathbb{E}\left[e^{\langle t,Y\rangle}\right]\left(\sum_{k=0}^M\frac{1}{2^k k!}\sum_{J\in S_{nd}^{(2k)}}\sum_{\alpha\in \mathcal{A}_J}t_{1}^{j_1}\dots t_{d}^{j_{n d}}\mathbb{E}\left[((A-K)^{\oplus n})_{\alpha_1,\alpha_2}\dots((A-K)^{\oplus n})_{\alpha_{2k-1},\alpha_{2k}}\right]+R_M\right),
    \end{multline*}
    recalling the definitions of $S_{nd}^{(2k)}$ in \eqref{def_S_k} and of $\mathcal{A}_J$ for $J\in S_{n d}^{(2k)}$ in \eqref{def_A_j}. From the last equality it immediately follows that for every $k\in\mathbb{N}$ one has that 
   \[
    \eta_{m}=
     \begin{cases}
    0 & \text{if $m\in S_{nd}^{(2k+1)}$}\\
    \frac{(2k)!}{2^k k! }\sum_{\alpha\in \mathcal{A}_m}\mathbb{E}\left[((A-K)^{\oplus n})_{\alpha_1,\alpha_2}\dots((A-K)^{\oplus n})_{\alpha_{2k-1},\alpha_{2k}}\right] & \text{if $m\in S_{nd}^{(2k)}$}.
    \end{cases}
    \]
    Therefore the Edgeworth expansion in \eqref{edg_standard} of order $2M$ reads as
    \begin{multline*}
   \sum_{k=0}^{M}\frac{1}{k!2^k}\sum_{J\in S_{nd}^{(2k)}}\sum_{\alpha\in \mathcal{A}_J}\mathbb{E}\left[\left((A-K)^{\oplus n}\right)_{\alpha_1,\alpha_2}\dots\left((A-K)^{\oplus n}\right)_{\alpha_{2k-1},\alpha_{2k}}\right]\frac{\partial^{2k}}{\partial x_{\alpha_1}\dots\partial x_{\alpha_{2k}}}\psi_Y(x)
   \\
    =\sum_{k=0}^{M}\frac{1}{k!2^k}\langle\mathbb{E}\left[\left((A-K)^{\oplus n}\right)^{\otimes k}\right],\nabla^{2k}\psi_Y(x)\rangle
    \\
    =\sum_{k=0}^{M}\frac{1}{k!2^k}\left\langle\mathbb{E}\left[\left((\sqrt{K}^{-1})^{\oplus n}(A-K)^{\oplus n}(\sqrt{K}^{-1})^{\oplus n}\right)^{\otimes k}\right],\frac{\left(\nabla^{2k}\phi_{I_{nd}}\left((\sqrt{K}^{-1})^{\oplus n}x\right)\right)}{\phi_{I_{nd}}\left((\sqrt{K}^{-1})^{\oplus n}x\right)}\right\rangle\psi_Y\left(x\right),
    \end{multline*}
    calling with $\phi_{I_{nd}}$ the density of a standard Gaussian distribution in $\mathbb{R}^{nd}$. Observing that $(\sqrt{K}^{-1})^{\oplus n}(A-K)^{\oplus n}(\sqrt{K}^{-1})^{\oplus n}=\left(\sqrt{K}^{-1}(A-K)\sqrt{K}^{-1}\right)^{\oplus n}$ and recalling the definitions of Hermite polynomials in \eqref{herm_def}, of the set $S_{n d}^{(2k)}$ for $k\in\mathbb{N}$ in \eqref{def_S_k} and of $\mathcal{A}_J$ for $J=(j_1,\dots,j_{nd})\in S_{nd}^{(2k)}$ in \eqref{def_A_j}, we obtain that the Edgeworth expansion can be written as 
    \begin{multline*}
    \sum_{k=0}^{M}\frac{1}{k!2^k}\sum_{J\in S_{nd}^{(2k)}}\sum_{\alpha\in \mathcal{A}_J}\mathbb{E}\left[(Q^{\oplus n})_{\alpha_1,\alpha_2}\dots(Q^{\oplus n})_{\alpha_{2k-1},\alpha_{2k}}\right]\cdot\\
    \cdot H_{j_1}\left(\left(\left(\sqrt{K}^{-1}\right)^{\oplus n}x\right)_1\right)\dots H_{j_{nd}}\left(\left(\left(\sqrt{K}^{-1}\right)^{\oplus n}x\right)_{nd}\right)\psi_Y\left(x\right)
    \end{multline*}
    \begin{multline}\label{edg_vera}
        =\sum_{k=0}^{M}\frac{1}{k!2^k}\sum_{J\in S_{n d}^{(2k)}}\sum_{\alpha\in \mathcal{A}_J}\mathbb{E}\left[(Q^{\oplus n})_{\alpha_1,\alpha_2}\dots(Q^{\oplus n})_{\alpha_{2k-1},\alpha_{2k}}\right]\cdot\\
    \cdot \prod_{i=0}^{n-1} \left(H_{j_{di+1}}\left(u^{(i)}_1\right)\dots H_{j_{di+d}}\left(u^{(i)}_{d}\right)\right)_{|_{u^{(i)}=\sqrt{K}^{-1}\tilde{x}_{i+1}}}\phi_K\left(\tilde{x}_{i+1}\right),
    \end{multline}
    using the notation as in \eqref{edg_measure} and calling $Q:=\sqrt{K}^{-1}(A-K)\sqrt{K}^{-1}$.
   
\end{example}

\section*{Acknowledgment}
{

The author was supported by the Luxembourg National Research Fund via the grant

\noindent PRIDE/21/16747448/MATHCODA. The author is grateful to Giovanni Peccati for guidance and support throughout this work.
}

\bibliographystyle{plainurl}
\bibliography{references}

@book{AS65,
  title={Handbook of Mathematical Functions: With Formulas, Graphs, and Mathematical Tables},
  author={Abramowitz, M. and Stegun, I.A.},
  isbn={9780486612720},
  lccn={lc65012253},
  series={Applied mathematics series},
  url={https://books.google.lu/books?id=MtU8uP7XMvoC},
  year={1965},
  publisher={Dover Publications}
}

@InProceedings{Ai20,
  title = 	 {Why bigger is not always better: on finite and infinite neural networks},
  author =       {Aitchison, Laurence},
  booktitle = 	 {Proceedings of the 37th International Conference on Machine Learning},
  pages = 	 {156--164},
  year = 	 {2020},
  editor = 	 {III, Hal Daumé and Singh, Aarti},
  volume = 	 {119},
  series = 	 {Proceedings of Machine Learning Research},
  month = 	 {13--18 Jul},
  publisher =    {PMLR},
  url = 	 {https://proceedings.mlr.press/v119/aitchison20a.html}
}

@misc{A19,
      title={Finite size corrections for neural network Gaussian processes}, 
      author={J. M. Antognini},
      year={2019},
      eprint={1908.10030},
      archivePrefix={arXiv},
      primaryClass={cs.LG},
      url={https://arxiv.org/abs/1908.10030}, 
}

@article{Torr23,
  title={Normal approximation of random gaussian neural networks},
  author={Apollonio, Nicola and De Canditiis, Daniela and Franzina, Giovanni and Stolfi, Paola and Torrisi, Giovanni Luca},
  journal={Stochastic Systems},
  volume={15},
  number={1},
  pages={88--110},
  year={2025},
  publisher={INFORMS}
}

@book{Apostol67,
  author    = {T. M. Apostol},
  title     = {Calculus, Volume 1: One-Variable Calculus, with an Introduction to Linear Algebra},
  publisher = {John Wiley \& Sons},
  year      = {1967},
  isbn      = {0-471-00005-1}
}

@article{PAPGGR23,
  title={A statistical mechanics framework for Bayesian deep neural networks beyond the infinite-width limit},
  author={Pacelli, R. and Ariosto, S. and Pastore, M. and Ginelli, F. and Gherardi, M. and Rotondo, P.},
  journal={Nature Machine Intelligence},
  volume={5},
  number={12},
  pages={1497--1507},
  year={2023},
  publisher={Nature Publishing Group UK London}
}

@misc{repo_edg_sim,
  author    = {Celli, L.},
  title     = {Edgworth Expansion for FCNNs (Simulations)},
  year      = {2026},
  doi       = {10.5281/zenodo.19737987},
  url       = {https://doi.org/10.5281/zenodo.19737987}
}

@misc{C26,
      title={Wide neural networks with general weights: convergence rate and explicit dependence on the hyper-parameters}, 
      author={L. Celli},
      year={2026},
      eprint={2601.21539},
      archivePrefix={arXiv},
      primaryClass={math.PR},
      url={https://arxiv.org/abs/2601.21539}, 
}

@misc{CP25,
      title={Entropic bounds for conditionally Gaussian vectors and applications to neural networks}, 
      author={L. Celli and G. Peccati},
      year={2025},
      eprint={2504.08335},
      archivePrefix={arXiv},
      primaryClass={math.PR},
      url={https://arxiv.org/abs/2504.08335}, 
}

@Article{Cybenko,
 Author = {Cybenko, G.},
 Title = {Approximation by superpositions of a sigmoidal function},
 FJournal = {MCSS. Mathematics of Control, Signals, and Systems},
 Journal = {Math. Control Signals Syst.},
 ISSN = {0932-4194},
 Volume = {2},
 Number = {4},
 Pages = {303--314},
 Year = {1989},
 Language = {English},
 DOI = {10.1007/BF02551274},
 zbMATH = {4114557},
 Zbl = {0679.94019}
}

@book{Ko06,
  title={Series Approximation Methods in Statistics},
  author={Kolassa, J.E.},
  isbn={9780387322278},
  lccn={97020556},
  series={Lecture Notes in Statistics},
  url={https://books.google.lu/books?id=aLVY_gVEomgC},
  year={2006},
  publisher={Springer New York}
}

@misc{BR25,
      title={Finite-Dimensional Gaussian Approximation for Deep Neural Networks: Universality in Random Weights}, 
      author={K. Balasubramanian and N. Ross},
      year={2025},
      eprint={2507.12686},
      archivePrefix={arXiv},
      primaryClass={stat.ML},
      url={https://arxiv.org/abs/2507.12686}, 
}

@article{BT24,
  author = {A. Basteri and D. Trevisan},
  year = {2024},
  title = {Quantitative Gaussian approximation of randomly initialized deep neural networks},
  journal = {Mach. Learn.},
  volume = {113},
  pages = {6373--6393}
}

@book{BR86,
  title={Normal Approximation and Asymptotic Expansions},
  author={Bhattacharya, R.N. and Rao, R.R.},
  isbn={9780898719444},
  series={Classics in Applied Mathematics},
  url={https://books.google.lu/books?id=H1lOIVHcRDEC},
  year={1986},
  publisher={Society for Industrial and Applied Mathematics}
}

@inproceedings{BFF24,
  author = {A. Bordino and S. Favaro and S. Fortini},
  year = {2024},
  title = {Non-asymptotic approximations of Gaussian neural networks via second-order Poincaré inequalities},
  booktitle = {Proceedings of Machine Learning Research (AABI24)}
}

@article{CC24,
author = {Carvalho, L. and Costa, J. L. and Mour\~{a}o, J. and Oliveira, G.},
title = {The Positivity of the Neural Tangent Kernel},
journal = {SIAM Journal on Mathematics of Data Science},
volume = {7},
number = {2},
pages = {495-515},
year = {2025},
doi = {10.1137/24M1659534},

URL = { https://doi.org/10.1137/24M1659534}
    }

@article{FHMNP,
  title={Quantitative clts in deep neural networks},
  author={Favaro, S. and Hanin, B. and Marinucci, D. and Nourdin, I. and Peccati, G.},
  journal={Probability Theory and Related Fields},
  volume={191},
  number={3},
  pages={933--977},
  year={2025},
  publisher={Springer}
}

@article{Fort22,
 arxiv = {https://arxiv.org/abs/2105.06868},
 author = {Fortuin, V.},
 journal = {International Statistical Review},
 number = {3},
 pages = {563--591},
 publisher = {Wiley Online Library},
 title = {Priors in Bayesian deep learning: A review},
 url = {https://onlinelibrary.wiley.com/doi/full/10.1111/insr.12502},
 volume = {90},
 year = {2022}
}

@book{H92,
  author    = {Hall, P.},
  title     = {The Bootstrap and Edgeworth Expansion},
  year      = {1992},
  publisher = {Springer},
  address   = {New York},
  isbn      = {978-0-387-97847-9}
}

@article{Han22,
  author = {B. Hanin},
  year = {2023},
  title = {Random neural networks in the infinite width limit as Gaussian processes},
  journal = {Ann. Appl. Probab.},
  volume = {33},
  number = {6A},
  pages = {4798--4819}
}

@article{Han_Gas,
  author = {B. Hanin},
  year = {2024},
  title = {Random Fully Connected Neural Networks as Perturbatively Solvable Hierarchies},
  journal = {Journal of Machine Learning Research}
}

@article{HBNPS,
  author       = {J. Hron and
                  Y. Bahri and
                  R. Novak and
                  J. Pennington and
                  J. Sohl{-}Dickstein},
  title        = {Exact Posterior Distributions of Wide Bayesian Neural Networks},
  journal      = {CoRR},
  volume       = {abs/2006.10541},
  year         = {2020},
  url          = {https://arxiv.org/abs/2006.10541},
  eprinttype    = {arXiv},
  eprint       = {2006.10541},
}

@inproceedings{Klu22,
  author = {A. Klukowski},
  year = {2022},
  title = {Rate of convergence of polynomial networks to Gaussian processes},
  booktitle = {Conference on Learning Theory, Proceedings of Machine Learning Research},
  pages = {701--722}
}

@inproceedings{LBNSPS,
  author = {J. Lee and Y. Bahri and R. Novak and S. Schoenholz and J. Pennington and J. Sohl-Dickstein},
  year = {2018},
  title = {Deep neural networks as Gaussian processes},
  booktitle = {International Conference on Learning Representation}
}

@article{Lu23,
  title = {Bayesian inference with finitely wide neural networks},
  author = {Lu, C.-K.},
  journal = {Phys. Rev. E},
  volume = {108},
  issue = {1},
  pages = {014311},
  numpages = {8},
  year = {2023},
  month = {Jul},
  publisher = {American Physical Society},
  doi = {10.1103/PhysRevE.108.014311},
  url = {https://link.aps.org/doi/10.1103/PhysRevE.108.014311}
}

@misc{MPS25,
      title={Edgeworth expansion on Wiener chaos}, 
      author={P. Mansanarez and G. Poly and Y. Swan},
      year={2025},
      eprint={2510.14002},
      archivePrefix={arXiv},
      primaryClass={math.PR},
      url={https://arxiv.org/abs/2510.14002}, 
}

@inproceedings{MHRTG,
  author = {A. Matthews and J. Hron and M. Rowland and R. Turner and Z. Ghahramani},
  year = {2018},
  title = {Gaussian process behaviour in wide deep neural networks},
  booktitle = {International Conference on Learning Representation}
}

@book{Mc87,
  author    = {McCullagh, P.},
  title     = {Tensor Methods in Statistics},
  year      = {1987},
  series    = {Monographs on Statistics and Applied Probability},
  publisher = {Chapman and Hall/CRC},
  doi       = {10.1201/9781351077118},
  url       = {https://doi.org/10.1201/9781351077118}
}

@article{NDSR20,
  title={Predicting the outputs of finite deep neural networks trained with noisy gradients.},
  author={G. Naveh and O. B. David and H. Sompolinsky and Z. Ringel},
  journal={Physical review. E},
  year={2020},
  volume={104 6-1},
  pages={
          064301
        },
  url={https://api.semanticscholar.org/CorpusID:238226559}
}

@book{Neal96,
  author = {R. Neal},
  year = {1996},
  title = {Bayesian learning for neural networks},
  volume = {118},
  publisher = {Springer}
}

@inproceedings{NO24,
title={Improving the Gaussian Approximation in Neural Networks: Para-Gaussians and Edgeworth Expansions},
author={M. Nica and J. Ortmann},
booktitle={NeurIPS 2024 Workshop on Mathematics of Modern Machine Learning},
year={2024},
url={https://openreview.net/forum?id=92q7WV4od7}
}

@book{NP12, place={Cambridge}, series={Cambridge Tracts in Mathematics}, title={Normal Approximations with Malliavin Calculus: From Stein’s Method to Universality}, publisher={Cambridge University Press}, author={Nourdin, Ivan and Peccati, Giovanni}, year={2012}, collection={Cambridge Tracts in Mathematics}}

@book{PT11,
  title={Wiener Chaos: Moments, Cumulants and Diagrams: A survey with Computer Implementation},
  author={Peccati, G. and Taqqu, M.S.},
  isbn={9788847016798},
  lccn={2010929482},
  series={Bocconi \& Springer Series},
  url={https://books.google.lu/books?id=qizrXkh1LrkC},
  year={2011},
  publisher={Springer Milan}
}

@article{PC21,
  title={The limitations of large width in neural networks: A deep Gaussian process perspective},
  author={Pleiss, G. and Cunningham, J. P.},
  journal={Advances in Neural Information Processing Systems},
  volume={34},
  pages={3349--3363},
  year={2021}
}

@book{GPML,
Title={Gaussian Processes for Machine Learning},
author={C. E. Rasmussen, C. K. I. Williams},
year={2006},
publisher={MIT Press,
ISBN	026218253X}
}

@InProceedings{SWG14,
  title = 	 {{Student-t Processes as Alternatives to Gaussian Processes}},
  author = 	 {Shah, A. and Wilson, A. and Ghahramani, Z.},
  booktitle = 	 {Proceedings of the Seventeenth International Conference on Artificial Intelligence and Statistics},
  pages = 	 {877--885},
  year = 	 {2014},
  editor = 	 {Kaski, Samuel and Corander, Jukka},
  volume = 	 {33},
  series = 	 {Proceedings of Machine Learning Research},
  address = 	 {Reykjavik, Iceland},
  month = 	 {22--25 Apr},
  publisher =    {PMLR},
  url = 	 {https://proceedings.mlr.press/v33/shah14.html}
}

@book{Stanley11,
place={Cambridge}, 
edition={2},
series={Cambridge Studies in Advanced Mathematics},
title={Enumerative Combinatorics},
publisher={Cambridge University Press},
author={Stanley, R. P.},
year={2011},
collection={Cambridge Studies in Advanced Mathematics}}

@misc{Trev,
  author = {D. Trevisan},
  year = {2023},
  title = {Wide deep neural networks with Gaussian weights are very close to Gaussian processes},
  note = {arXiv:2312.06092 [math.ST]}
}

@book{We13,
  title={Asymptotic Expansions for General Statistical Models},
  author={Wefelmeyer, W. and Pfanzagl, J.},
  isbn={9781461564799},
  series={Lecture Notes in Statistics},
  url={https://books.google.lu/books?id=L14FCAAAQBAJ},
  year={2013},
  publisher={Springer New York}
}
\end{document}